\documentclass[letterpaper]{article} 
\usepackage{aaai2026}
\usepackage{times}  
\usepackage{helvet}  
\usepackage{courier}  
\usepackage[hyphens]{url}  
\usepackage{graphicx} 
\usepackage{natbib}  
\usepackage{caption} 
\usepackage{comment}
\usepackage{algorithm}
\usepackage{algorithmic}

\usepackage{newfloat}
\usepackage{listings}
\DeclareCaptionStyle{ruled}{labelfont=normalfont,labelsep=colon,strut=off} 
\floatstyle{ruled}
\newfloat{listing}{tb}{lst}{}
\floatname{listing}{Listing}
\title{The Correspondence Between Bounded Graph Neural Networks and Fragments of First-Order Logic}
\author{
Bernardo Cuenca Grau$^1$,
Eva Feng$^1$,
Przemys\l aw A. Wa\l\k{e}ga$^2$
}
\affiliations{
$^1$ Department of Computer Science, University of Oxford\\
$^2$ 
Queen Mary University of London, University of Łódź
\\
bernardo.grau@cs.ox.ac.uk, eva.feng@cs.ox.ac.uk, p.walega@qmul.ac.uk
}

\usepackage{soul}
\usepackage[utf8]{inputenc}
\usepackage{enumitem}
\usepackage{amsmath}
\usepackage{amsfonts}
\usepackage{amsthm}
\usepackage{stix2}
\usepackage{booktabs}
\usepackage{multirow}
\usepackage{tikz}
\usepackage{cleveref}
\usetikzlibrary{shapes.geometric, positioning, fit, arrows}
\usepackage{todonotes}
\newcommand{\bcgline}[1]{\todo[author=Bernardo,inline,backgroundcolor=orange!20]{#1}}

\newcommand{\R}{\ensuremath{\mathbb{R}}}

\newcommand{\true}{\ensuremath{\mathsf{true}}}
\newcommand{\false}{\ensuremath{\mathsf{false}}}
\newcommand{\prop}{\ensuremath{\mathsf{PROP}}}
\newcommand{\M}{\ensuremath{\mathfrak{M}}}

\newcommand{\modE}{\ensuremath{\mathsf{E}}}
\newcommand{\modA}{\ensuremath{\mathsf{A}}}

\newcommand{\countE}{\ensuremath{\mathsf{\exists}}}
\newcommand{\Ecount}{\ensuremath{\mathsf{\exists}}\#}

\newcommand{\FO}{\ensuremath{\text{FO}}}
\newcommand{\FOtwo}{\ensuremath{\text{FO}^2}}
\newcommand{\Ctwo}{\ensuremath{\text{C}^2}}
\newcommand{\ML}{\ensuremath{\mathcal{ML}}}

\newcommand{\GML}{\ensuremath{\mathcal{GML}}}
\newcommand{\GMLC}{\ensuremath{\mathcal{GMLC}}}
\newcommand{\MLA}{\ensuremath{\mathcal{ML}(\E)}}
\newcommand{\MLE}{\ensuremath{\mathcal{ML}(\E)}}
\newcommand{\MLpar}{\ensuremath{\mathcal{ML}^{\neg,\cap,\cup,-,id}}}
\newcommand{\EMLC}{\ensuremath{\mathcal{EMLC}}}
\newcommand{\EML}{\ensuremath{\mathcal{EML}}}

\newcommand{\G}{\ensuremath{G}}
\newcommand{\V}{\ensuremath{V}}
\newcommand{\E}{\ensuremath{\mathsf{E}}}
\newcommand{\emb}{\ensuremath{\lambda}}

\newcommand{\N}{\ensuremath{\mathcal{N}}}
\newcommand{\Lr}[1][]{%
  \ifx#1\empty
    \ensuremath{\mathcal{L}}%
  \else
    \ensuremath{\mathcal{L}^{#1}}%
  \fi
}   
\newcommand{\comb}{\ensuremath{\mathsf{comb}}}
\newcommand{\agg}{\ensuremath{\mathsf{agg}}}
\newcommand{\aggn}{\ensuremath{\mathsf{\overline{agg}}}}
\newcommand{\readout}{\ensuremath{\mathsf{read}}}
\newcommand{\cls}{\ensuremath{\mathsf{cls}}}

\newcommand{\s}{\ensuremath{\mathsf{s}}}
\newcommand{\bm}{\ensuremath{\mathsf{b}}}
\newcommand{\m}{\ensuremath{\mathsf{m}}}

\newcommand{\ac}{\ensuremath{\mathsf{AC}}}
\newcommand{\acr}{\ensuremath{\mathsf{ACR}}}

\newcommand{\acn}{\ensuremath{\mathsf{AC+}}}

\newcommand{\acar}{\ensuremath{\mathsf{AC/RC}}}

\newcommand{\GNN}[2]{\ensuremath{\text{GNN}_{#1}^{#2}}}

\newcommand{\Nn}{\ensuremath{\overline{N}}}

\usepackage{thm-restate}
\declaretheorem[name=Theorem]{thm}
\newtheorem{theorem}[thm]{Theorem}
\newtheorem{example}[thm]{Example}
\newtheorem*{example*}{Example}
\newtheorem{definition}[thm]{Definition}

\newtheorem{corollary}[thm]{Corollary}
\newtheorem*{claim*}{Claim}
\newtheorem*{theorem*}{Theorem}

\begin{document}

\maketitle

\begin{abstract}
Graph Neural Networks (GNNs) address two key challenges in applying deep learning to graph-structured data: they handle varying size input graphs and ensure invariance under graph isomorphism. While GNNs have demonstrated broad applicability, understanding their expressive power remains an important question.
In this paper, we propose GNN architectures that correspond precisely to prominent fragments of first-order logic ($\FO$), including various modal logics as well as more expressive two-variable fragments. To establish these results, we apply methods from finite model theory of first-order and modal logics to the domain of graph representation learning.
Our results provide a unifying framework for understanding the logical expressiveness of GNNs within $\FO$.
\end{abstract}


 \begin{links}
     \link{Extended version}{https://arxiv.org/abs/2505.08021}
\end{links}

\section{Introduction}

Learning on graphs 
or relational structures presents two fundamental challenges. First,  neural networks require fixed size inputs, making them ill-suited for graphs of varying size. Second, predictions about graphs should not depend on how the graph is represented, i.e., they should be invariant under isomorphism \cite{DBLP:series/synthesis/2020Hamilton}. 

\emph{Graph Neural Networks} (GNNs) \cite{DBLP:conf/icml/GilmerSRVD17}
overcome these limitations by operating natively on graph-structured data, inherently handling variable sizes and ensuring representation invariance.
The flagship \emph{aggregate-combine} (\ac) architecture can be viewed as a layered network operating over an input graph. Each node maintains a state (a real-valued vector) and,
in each layer, a node’s state is updated based on its current state and that of its  neighbours. 
This update mechanism is specified by an \emph{aggregation function} that takes the current states of the neighbours and aggregates them into a vector, and a \emph{combination function} that takes the aggregate value from the neighbours and the current state of the node and computes the updated state.
Their \emph{aggregate-combine-readout} (\acr)
extension includes an additional \emph{readout function} which aggregates states across all nodes in the graph, rather than just local neighbours \cite{DBLP:conf/iclr/BarceloKM0RS20}. 
GNNs have been widely applied. They drive recommendation systems  \cite{Ying_2018},  predict molecular properties \cite{DBLP:journals/air/BesharatifardV24},  enhance traffic navigation  \cite{Derrow_Pinion_2021}, interpret  scenes in computer vision \cite{DBLP:journals/corr/abs-2209-13232}, and enable  reasoning over incomplete knowledge graphs \cite{DBLP:conf/iclr/CucalaGKM22,DBLP:conf/nips/ZhangC18,DBLP:journals/corr/abs-2402-04062}.

GNNs encompass many architectures and a central question is understanding their \emph{expressive power}---i.e., the classes of functions they can compute.
This has been addressed from multiple angles.
Early works studied the \emph{discriminative power} of GNNs: given two graphs, can a GNN from a given family yield distinct outputs for them? 
By design, no GNN can separate isomorphic graphs, but more subtly, certain non-isomorphic graphs may remain indistinguishable. In particular, if two graphs cannot be distinguished by the 1-dimensional Weisfeiler-Leman (WL)  graph isomorphism test, then no  GNN can differentiate them either~\cite{morris2019weisfeiler,DBLP:conf/iclr/XuHLJ19}. 
Generalised k-dimensional GNNs, which handle higher-order graph structures, have also been connected to the WL hierarchy of increasingly powerful isomorphism tests \cite{morris2019weisfeiler}.
Through the correspondence between WL  and finite-variable  logics \cite{cai1992optimal}, the limitation extends to logical distinguishability.

The expressiveness of GNNs has also been studied through the lenses of database query languages. 
As node classifiers, GNNs compute a \emph{unary query}---an isomorphism-invariant function mapping each graph 
and node to a truth value. 
For a family of GNN classifiers, what is the  logic expressing these unary queries? This is  the \emph{logical expressiveness} (or \emph{uniform expressiveness}) of GNNs. 
The expressiveness of GNNs  goes beyond first-order logic (\FO) since aggregation can only be captured using extensions such as  counting terms \cite{grohe2024descriptive,DBLP:conf/nips/Huang0CB23},  Presburger quantifiers \cite{benedikt_et_al:LIPIcs.ICALP.2024.127}, or linear programming \cite{DBLP:conf/ijcai/NunnSST24}.
Other GNN variants, such as recursive GNNs \cite{DBLP:conf/nips/AhvonenHKL24,DBLP:conf/aaai/PfluegerCK24} require fixpoint operators.

A connection between  GNNs and \FO{} fragments has also been established \cite{DBLP:conf/iclr/BarceloKM0RS20}.  
Graded modal logic (\GML) formulas, a.k.a.\ concepts in the description logic $\mathcal{ALCQ}$ \cite{DBLP:conf/dlog/2003handbook}, can be captured by GNNs without readout functions, 
just as $\FO$ formulas with two variables and counting quantifiers ($\Ctwo$)  can be realised by a GNNs with readouts. 
This relationship is, however, asymmetric: while any \GML{} classifier can be expressed by an \ac{} GNN, the converse requires an assumption of \FO{} expressibility.  
The conditions ensuring \FO{} expressibility of a GNN remain largely unexplored:  
the only sufficient condition known to us is that monotonic GNNs with max aggregation precisely match unions of tree-shaped conjunctive queries \cite{DBLP:conf/kr/CucalaGMK23}.

\smallskip
\noindent
\textbf{Contributions} 
We introduce \emph{bounded GNNs} with  \emph{$k$-bounded aggregation}, where multiplicities greater than $k$ in a multiset are capped at $k$. 
If $k=1$, the multiplicities do not matter, and we speak of \emph{set-based aggregation}.
As we show, bounded GNNs correspond to modal and two-variable \FO{} fragments as depicted in \Cref{fig:landscape}.

We first establish that $\ac$ GNNs with set-based aggregation ($\GNN{\s}{\ac}$) correspond to basic modal logic ($\ML$) and thus to concepts in the description logic $\mathcal{ALC}$. 
This extends to  GNNs using bounded aggregation  ($\GNN{\bm}{\ac}$), which correspond to graded modal logic ($\GML$), that is, concepts of $\mathcal{ALCQ}$.
Readouts enable global quantification: GNNs with set-based aggregation and readout ($\GNN{\s}{\acr}$) capture modal logic with the global modality ($\MLA$), which corresponds to  $\mathcal{ALC}$ with the universal role \cite{DBLP:conf/dlog/2003handbook}, while those with bounded aggregation and readouts ($\GNN{\bm}{\acr}$) match graded modal logic with counting $(\GMLC)$, which corresponds to  $\mathcal{ALCQ}$ equipped with the universal role. 
While bounded readouts enable global quantification, they cannot express certain first-order properties like ``nodes with exactly $k$ non-neighbours".
To overcome this limitation, we introduce $\GNN{\bm}{\acn}$: a family of bounded GNNs 
augmented with an aggregation function over non-neighbours. We prove that $\GNN{\bm}{\acn}$ 
captures $\Ctwo{}$ (two-variable FO with counting), while its set-based variant $\GNN{\s}{\acn}$ corresponds to the  two-variable \FO{} fragment $\FOtwo{}$.








\begin{figure}[t]
\centering
\begin{tikzpicture}[scale=1, myellipse/.style 2 args={ellipse, fill=black!#1, draw, fill opacity=0.3}]

\scriptsize

\node[myellipse={10}{\FO}, 
label={[anchor=north, below=0.9mm]\FO},
minimum width=6.1cm, minimum height=5.5cm, draw] (e1) {};

\node[myellipse={20}{\Ctwo},
label={[anchor=north, below=0.05cm]$\GNN{\bm}{\acn} \equiv \Ctwo$},
minimum width=5.9cm, minimum height=5.0cm, above=-0mm of e1.south] (e2) {};

\node[myellipse={30}{\FOtwo},
label={[below=0.3mm]{\parbox{2cm}{\centering
$\GNN{\s}{\acn}$\\
\rotatebox{90}{$\equiv$}\\
$\FOtwo$}}},
xshift=0cm, yshift=-0cm, minimum width=2.4cm, minimum height=4.5cm, above=0mm of e1.south] (e3) {};

\node[myellipse={40}{\GMLC},
label={[below left=8mm and 7mm]{\parbox{2cm}{\centering
$\GNN{\bm}{\acr}$\\
\rotatebox{90}{$\equiv$}\\
$\GMLC$}}},
xshift=0cm, yshift=0cm, minimum width=5.1cm, minimum height=3.4cm, above=0mm of e1.south] (e4) {};


\node[myellipse={50}{\GML},
label={[below right=-1mm and 1mm]{\parbox{2cm}{\centering
$\GNN{\bm}{\ac}$\\
\rotatebox{90}{$\equiv$}\\
$\GML$}}},
xshift=1.25cm, yshift=0.2cm, rotate=20, minimum width=3.2cm, minimum height=1.7cm, above=0mm of e1.south] (e6) {};

\node[myellipse={50}{\MLA},
label={[below=0.8mm]{\parbox{2cm}{\centering
$\GNN{\s}{\acr}$\\
\rotatebox{90}{$\equiv$}\\
$\MLA$}}},
xshift=0cm, yshift=0cm, minimum width=1.9cm, minimum height=2.8cm, above=0mm of e1.south] (e7) {};

\node[ellipse, draw, fill=black!50,
label={[below=1mm]{\parbox{2cm}{\centering
$\GNN{\s}{\ac}$\\
\rotatebox{90}{$\equiv$}\\
$\ML$}}},
xshift=0cm, yshift=0cm, minimum width=1.2cm, minimum height=1.2cm, above=0mm of e1.south] (e8) {};
\end{tikzpicture}
\caption{The landscape of our expressive power results}
\label{fig:landscape}
\end{figure}


This paper comes with an appendix containing all the proofs of our technical results.

\section{Graphs and Classifiers}\label{classifiers}


\noindent \textbf{Graphs}
We consider \emph{(finite, undirected, simple, and node-labelled) graphs} $\G = (V, E, \emb)$, where $V$ is a finite set of nodes, $E$ a set of undirected edges with no self-loops, and $\emb:V \to \{0,1 \}^d$ assigns to each node a 
binary vector\footnote{The assumption that node labels are binary, i.e.,  nodes are coloured, is standard when studying logical characterisation of GNNs \cite{DBLP:conf/iclr/BarceloKM0RS20,benedikt_et_al:LIPIcs.ICALP.2024.127,DBLP:conf/ijcai/NunnSST24}} of dimension $d$.
The dimension $d$ of all vectors in $\G$ is the same, and we refer to it as the \emph{dimension} of $\G$. 
A \emph{pointed graph} is a pair $(\G,v)$ of a graph and one of its nodes.

\smallskip
\noindent \textbf{Node Classifiers}
A \emph{node classifier}
is a function  mapping
pointed graphs  to  $\true$ or $\false$.  The classifier accepts the input if it returns $\true$ and it rejects if it returns $\false$.
Two classifiers are \emph{equivalent} if they compute the same function. 
A family $\mathcal{F}$ of classifiers is \emph{at most as expressive} as $\mathcal{F}'$, written $\mathcal{F} \leq \mathcal{F}'$, 
if each classifier in $\mathcal{F}$ has an equivalent one in $\mathcal{F}'$. 
If $\mathcal{F} \leq \mathcal{F}'$  and $\mathcal{F}' \leq \mathcal{F}$, we write $\mathcal{F} \equiv \mathcal{F}'$, and say that $\mathcal{F}$ and $\mathcal{F}'$ have the \emph{same expressiveness}. 
Such \emph{uniform expressiveness} contrasts with  other notions such as discriminative power \cite{morris2019weisfeiler,DBLP:conf/icml/WangZ22} and non-uniform expressiveness  \cite{grohe2024descriptive}.

\smallskip
\noindent \textbf{GNN Classifiers}
We consider standard GNNs with
 \emph{aggregate-combine} (\ac) and  \emph{aggregate-combine-readout} (\acr) layers \cite{benedikt_et_al:LIPIcs.ICALP.2024.127,DBLP:conf/iclr/BarceloKM0RS20}, and propose also \emph{extended aggregate-combine} (\acn) layers equipped with one aggregation over neighbours and another  over non-neighbours.
An \ac{} layer is a pair $( \agg, \comb )$, an \acr{} layer is a triple $( \agg, \comb, \readout)$ and an \acn{} layer is a triple $( \agg, \aggn, \comb)$, where  $\agg$ and $\aggn$ are \emph{aggregation functions} and $\readout$ is a \emph{readout function}, all  mapping multisets of vectors into single vectors, whereas $\comb$ is  a \emph{combination function} mapping  vectors to  vectors.
An application of a layer to a graph $\G = ( V, E, \emb )$ yields a graph $\G' = ( V, E, \emb' )$ with the same nodes and edges, but with an updated labelling function 
$\emb'$.
For an $\ac$ layer, vector 
$\emb'(v)$ is defined as follows
for each node  $v$, where $N_\G(v) = \{w \mid \{u,w \} \in E \}$ is the set of 
neighbours of $v$ and 
$\overline{N}_\G(v) = \{w \mid \{u,w \} \not\in E \} \backslash \{v\}$ is the set of non-neighbours of $v$ excluding $v$ itself (note that the graphs we consider have no self loops).
\begin{eqnarray}
\comb \Big( \emb(v), \agg( \lBrace  \emb(w) \rBrace_{w \in N_G(v)} )  \Big). && \label{eq:ac}
\end{eqnarray}
%
For an $\acr$ layer, vector $\emb'(v)$ is defined as
\begin{eqnarray}
\comb \Big( \emb(v), 
\agg( \lBrace  \emb(w) \rBrace_{w \in N_G(v)} ), \readout( \lBrace  \emb(w) \rBrace_{w \in V}) \Big).  && \label{eq:integrated}
\end{eqnarray}
In turn, for an $\acn$ layer, vector $\emb'(v)$ is defined as
\begin{eqnarray}
\comb \Big( \emb(v), 
\agg( \lBrace  \emb(w) \rBrace_{w \in N_G(v)} ), \aggn ( \lBrace  \emb(w) \rBrace_{w \in \Nn_G(v)} )\Big).  && \label{eq:acn}
\end{eqnarray}
Each  $\agg$, $\aggn$, $\comb$, and $\readout$ has domain and range of some fixed dimension (but each 
can have a different dimension), which we refer to as
input and output dimensions.
For layer application to be meaningful, these dimensions need to match: if in an $\ac$  layer  $\agg$  has input dimension  $d$ and output dimension $d'$, then the input dimension of $\comb$ is $d+d'$; in an $\acr$ layer, if  $\agg$ has dimensions $d$ and $d'$, and  $\readout$ has dimensions $d$ and $d''$ (note that input dimensions of $\agg$ and $\readout$ need to match),  the input dimension of $\comb$ is $d+d'+d''$; if in an $\acn$ layer, $\agg$  has input dimension  $d$ and output dimension $d'$, and  $\aggn$ has dimensions $d$ and $d''$ , then the input dimension of $\comb$ is $d+d'+d''$.
A \emph{GNN classifier} $\N$ of dimension $d$ consists of $L$ layers and a 
classification function $\mathsf{cls}$ from vectors to truth values.
The input dimension of  the first layer  is $d$ and consecutive
layers have matching dimensions: the output dimension of layer $i$ matches the input dimension of  layer $i+1$.
We write $\lambda(v)^{(\ell)}$ for the vector of node $v$ upon application of layer $\ell$;  $\lambda(v)^{(0)}$ is the initial label of $v$, and $\lambda(v)^{(L)}$ is its final label. 
The application of  $\N$ 
to  $(\G,v)$ is  the truth value  $\N(\G,v) =\cls(\emb(v)^{(L)})$.

\smallskip
\noindent \textbf{Logic Classifiers}
We consider  formulas over finite signatures consisting of a set $\prop$ of propositions (unary predicates)
$p_1, p_2, \dots$ 
for node colours. Formulas of the
\emph{graded modal logic of counting} ($\GMLC$) are defined as follows: 
$$
\varphi :=  p \mid \neg \varphi \mid \varphi \land \varphi \mid \Diamond_{k} \varphi  \mid \countE_k \varphi,  
$$
where $p \in \prop$ and, for each $k \in \mathbb{N}$, $\Diamond_k$ and $\countE_k$ are the $k$-graded modality and the $k$-counting modality, respectively.
The formula $\Diamond_{k} \varphi$ expresses that at least $k$ accessible worlds satisfy 
$\varphi$, while $\countE_k \varphi$ states that at least $k$ worlds in total satisfy 
$\varphi$. Graded modal logic ($\GML$) is obtained from $\GMLC$ by disallowing counting 
modalities. Modal logic with the global modality ($\MLA$) is obtained from $\GMLC$ by 
restricting both counting and graded modalities to  $k=1$. Basic modal logic ($\ML$) further restricts $\MLA$ by disallowing counting modalities entirely. 
We also consider the two-variable fragment of first-order logic with counting quantifiers ($\Ctwo$), where $\countE_k$ denotes counting quantifiers.\footnote{We use the symbols $\countE_k$ 
for both counting 
quantifiers and counting modalities.} The classical two-variable fragment ($\FOtwo$) is obtained from $\Ctwo$ by restricting counting quantifiers to $k=1$.


The  \emph{depth} of formula $\varphi$
the maximum nesting of modal operators ($\Diamond_k$ and $\countE_k$) or quantifiers in it.  
The  \emph{counting rank},  $\mathsf{rk}_{\#}(\varphi)$, is
the maximal  among numbers $k$ occurring in its graded and counting modalities ($\Diamond_k$ and $\countE_k$) or in counting quantifiers, or $0$ if the formula does not mention any modalities or quantifiers. 
For $\mathcal{L}$ any of the logics defined above, we denote as $\mathcal{L}_{\ell,c}$
the set of all $\mathcal{L}$ formulas of depth at most $\ell$ and counting rank at most $c$.


Formulas are evaluated over pointed models $(\M_{\G},v)$, each corresponding to a pair of a (coloured) graph $\G = (V, E, \emb)$ of some dimension $d$ and a node $v\in V$.
In the case of modal logics,
model $\M_{\G} = (V, E, \nu)$ has $V$ as the set of modal worlds, $E$ as the symmetric accessibility relation, and the valuation function $\nu$ maps each $p_i \in \prop$ to the subset of nodes in $V$ whose vectors have $1$ on the $i$-th position. 
Valuation $\nu$ extends to all modal formulas: 
\begin{align*}
\nu(\neg \varphi) & := V \setminus \nu(\varphi), \qquad 
\nu(\varphi_1 \land \varphi_2)  := \nu(\varphi_1) \cap \nu(\varphi_2), \\
\nu(\Diamond_k \varphi) & := \{ v \mid
 k \leq | \{ w \mid \{ v,w \} \in E \text{ and } w \in \nu(\varphi) \}| \}, \\
\nu(\countE_k \varphi) & :=  V \text{ if } k \leq |\nu(\varphi)|, \text{ and } \emptyset  \text{ otherwise}.
\end{align*}
%
%
A pointed model $(\mathfrak{M}, v)$, 
\emph{satisfies} a formula $\varphi$, denoted $(\mathfrak{M},v) \models \varphi$, if
$v \in \nu(\varphi)$.
In the case of \Ctwo{} formulas, we treat $\M_{\G}$ as the corresponding \FO{} structure, and evaluate formulas using the standard \FO{} semantics.

For a  logic $\mathcal{L}$, we write $(\mathfrak{M},w) \equiv_{\mathcal{L}} (\mathfrak{M}',w')$ if $(\mathfrak{M},w)$ and $(\mathfrak{M}',w')$ satisfy the same formulas of $\mathcal{L}$. 
A \emph{logic  classifier} of dimension $d$ is a formula $\varphi$ 
with at most $d$ propositions (unary predicates) $p_1, \dots, p_d$. 
The application of $\varphi$ to $(\G,v)$ is $\true$ if $(\M_{\G},v) \models \varphi$ and $\false$ otherwise.
By  convention, we use the same symbol for a logic and its associated classifier family.

\section{Bounded GNN Classifiers}\label{boundedGNN}


We next introduce \emph{bounded GNNs}, which generalise max GNNs \cite{DBLP:conf/kr/CucalaG24}
and max-sum GNNs \cite{DBLP:conf/kr/CucalaGMK23}.
Bounded GNNs restrict aggregation and readout  by requiring existence  of a bound $k$ such that all multiplicities $k' >k$ in an input multiset are replaced with $k$. Thus, multiplicities greater than $k$ do not affect the output of a $k$-bounded function.
Set-based functions ignore multiplicities altogether.

\begin{definition}
An aggregation (or readout) function $f$ 
is $k$-\emph{bounded}, for  $k \in \mathbb{N}$,  if $f(M) = f(M_{k})$ for each multiset $M$ in the domain of $f$,
where $M_{k}$ is the multiset obtained from $M$ by replacing 
all multiplicities greater than $k$ with $k$.
Function $f$ is  
\emph{set-based} if it is $1$-bounded, and it is \emph{bounded} if it is $k$-bounded for some $k \in \mathbb{N}$.
\end{definition}


\begin{example}\label{bounded_functions}
Consider the  example functions below.
\begin{itemize}
\item The  aggregation in max GNNs \cite{DBLP:conf/kr/CucalaG24}
 is set-based.
It maps a multiset of vectors to a vector being their componentwise maximum,
for instance $\lBrace (3,2),(2,4), (2,4) \rBrace \mapsto (3,4)$.

\item  The aggregation  in max-k-sum GNNs \cite{DBLP:conf/kr/CucalaGMK23} is $k$-bounded.
It maps $M$ to a vector whose $i$th component is the sum of the 
$k$ largest 
$i$th components in $M$;
if $k=2$,  $\lBrace (3,2),(2,4),(2,4)\rBrace \mapsto (5,8)$.

\item Examples of unbounded functions include componentwise sum $\lBrace (3,2),(2,4),(2,4)\rBrace \mapsto (7,10)$ and the average mapping $\lBrace (3,2),(2,4),(2,4) \rBrace \mapsto (\frac{7}{3}, \frac{10}{3})$.
\end{itemize}
\end{example}

Equipped with the notion of bounded aggregation and readout functions, we are ready to define bounded GNNs.

\begin{definition}\label{def:families}
We consider families of GNN classifiers, $\GNN{X}{Y}$, where $X \in \{\s, \bm, \m\}$ indicates the type of aggregation and readout:  
set-based ($\s$),
bounded  ($\bm$), or arbitrary---also called 
multiset---($\m$), whereas  $Y \in \{\ac, \acr, \acn \}$ indicates whether the GNN uses only $\ac$, $\acr$, or $\acn$ layers. 
\emph{Bounded GNN classifiers} are those with bounded  aggregation and readout functions.
\end{definition}

All  $\GNN{X}{Y}$ classifiers, with $X \in \{ \s, \bm \}$, are bounded.
Family  $\GNN{\m}{\ac}$ corresponds to aggregate-combine GNNs, 
$\GNN{\m}{\acr}$ to  aggregate-combine-readout GNNs
\cite{DBLP:conf/iclr/BarceloKM0RS20},
$\GNN{\bm}{\ac}$ contains 
monotonic max-sum GNNs \cite{DBLP:conf/kr/CucalaGMK23}, and
$\GNN{\s}{\ac}$ 
contains 
max GNNs \cite{DBLP:conf/kr/CucalaG24}. 
%




As shown later, the expressiveness of  bounded GNN classifiers falls within $\FO$.
This is intuitively so, because bounded GNNs have finite spectra, as defined below.

\begin{definition}\cite{benedikt_et_al:LIPIcs.ICALP.2024.127}
The \emph{spectrum}, $\mathsf{sp}(\N)$,  of a GNN classifier $\N$ (of dimension $d$), is the set of all vectors that can occur as node labels in any layer of  $\N$ application  (to graphs of dimension $d$). 
For $L$ the number of layers of $\N$ and $\ell \leq L$, we let  $\mathsf{sp}(\N,\ell)$ be the subset of the spectrum consisting of the vectors that can occur upon application of layer $\ell$. 
By convention, we let $\mathsf{sp}(\N,0)$ be the set of Boolean vectors of the classifier's dimension.
\end{definition}

Since $\mathsf{sp}(\N,0)$
is always finite and bounded functions applied to multisets with a bounded number of vectors yield finitely many possible outcomes, the 
spectra of bounded GNN classifiers are bounded. Using combinatorial arguments we can obtain explicit bounds as below.

\begin{restatable}{proposition}{finitesp}\label{prop:spectrum}
Each bounded GNN  classifier $\N$  has a finite spectrum.
In particular, 
if $\N$ has dimension $d$, $L$ layers, and $k$ is the largest bound of its aggregation and readout functions, then
$|\mathsf{sp}(\N,0)| = 2^d$ and  $| \mathsf{sp}(\N,\ell+1)|$ , for each $0 \leq \ell \leq L - 1$, is bounded by the following values:
\begin{align*}
&  
|\mathsf{sp}(\N, \ell)| \cdot
(k+1)^{ |\mathsf{sp}(\N, \ell)|}, && \text{if } \N \in \GNN{\bm}{\ac},
\\
&  
 |\mathsf{sp}(\N, \ell)| \cdot
(k+1)^{2  |\mathsf{sp}(\N, \ell)|} ,&&  \text{if } \N \in \{\GNN{\bm}{\acr}, \GNN{\bm}{\acn}\} .
\end{align*}
\end{restatable}

\section{Overview and Technical Approach}\label{sec:overview}

In what follows, we systematically establish the correspondences between bounded GNNs and logic 
classifiers depicted in Figure~\ref{fig:landscape}. 
In \Cref{sec:agg-combine}, we show that bounded aggregate-combine GNNs correspond precisely to the modal logics without global counting ($\ML$ and $\GML$). Next, in \Cref{sec:readouts}, we show that bounded aggregate-combine-readout GNNs capture the expressive power of modal logics with global counting ($\MLA$ and $\GMLC$). Finally, in \Cref{sec:two-var}, we prove that extended aggregate-combine GNNs are equivalent in expressive power to two-variable logics ($\Ctwo{}$ and $\FOtwo{})$.

To characterise a family of GNN classifiers  $\mathcal{F}$ via a logic $\mathcal{L}$ we establish a bidirectional correspondence. We first show that every formula of $\mathcal{L}$ can be simulated by a GNN in $\mathcal{F}$. 
We then show the converse:  every GNN classifier in $\mathcal{F}$ admits an equivalent $\mathcal{L}$-classifier.

The second step builds on finite model theory 
characterisations of $\mathcal{L}$-equivalence through model comparison games (independent of GNNs). The existence of a winning strategy
induces an equivalence relation $\sim$ on pointed models, which extends to
pointed coloured graphs: $(G,v) \sim (G',v')$ holds precisely when $(\M_{G},v) \sim (\M_{G'},v')$.
By limiting games to a fixed number of rounds and bounded grading, and by considering finite signatures, we ensure that $\sim$-invariant classes 
of  models  can be represented as a finite disjunction of 
\emph{characteristic formulas of $\mathcal{L}$} \cite{DBLP:journals/corr/abs-1910-00039, DBLP:books/sp/Libkin04}. 
Specifically, for modal logics and two-variable fragments we use suitable variants of bisimulation and 2-pebble games, respectively.

The final requirement to establish the connection to GNNs is to show that each GNN classifier $\mathcal{N}$   
in $\mathcal{F}$ is invariant under $\sim$, i.e., 
$(G_1,v_1) \sim (G_2, v_2)$ implies $\N(G_1,v_1) = \N(G_2,v_2)$,
for all pointed graphs $(G_1,v_1)$ and $(G_2,v_2)$.
%

\section{Modal Logics Without Global Counting}\label{sec:agg-combine}

We first study bounded aggregate-combine GNNs, and start by showing that \GML{} and \ML{} formulas can be captured by $\GNN{\bm}{\ac}$  and  $\GNN{\s}{\ac}$ classifiers, respectively.
For this, we adapt the construction simulating \GML{} formulas with GNNs which uses unbounded summation \cite{DBLP:conf/iclr/BarceloKM0RS20}.
For \GML{}, our GNN construction uses max-$k$-sum aggregation, and for \ML{} it uses max aggregation (which coincides with max-$k$-sum, for $k=1$).

\begin{restatable}{theorem}{linalgac}\label{prop:ml-to-gnn}
$\GML \leq \GNN{\bm}{\ac}$ and $\ML \leq  \GNN{\s}{\ac}$.
\end{restatable}

\begin{proof}[Proof Sketch]
Let $\varphi \in \GML$ be a  logic classifier of dimension $d$ with subformulas $\varphi_1, \ldots, \varphi_L$, such that $k \leq \ell$ if $\varphi_k$ is a subformula of $\varphi_{\ell}$.
We construct $\N_{\varphi}$ with layers $0, \ldots, L$ and
a classification function that maps a vector to \true{} iff its last element is $1$.
Layer $0$  multiplies input vectors by a  matrix $\mathbf{D} \in \R^{d \times L}$, namely  $\lambda(v)^{(1)} =  \lambda(v)^{(0)}\mathbf{D} $, where $D_{k \ell} = 1$ if the $k$th position of the  input vectors  corresponds to a proposition $\varphi_\ell$;
other entries of $\mathbf{D}$ are $0$.
All other layers are \ac{} layers of dimension $L$ using  max-$n$-sum, for $n$ the counting rank of $\varphi$. The combination function is $\comb(\mathbf{x}, \mathbf{y}) = \sigma(\mathbf{x} \mathbf{C}  + \mathbf{y} \mathbf{A} + \mathbf{b})$, where 
$\sigma(x) = \mathit{min}(\mathit{max}(0,x),1)$ is the truncated ReLU and where entries of matrices $\mathbf{A}, \mathbf{C} \in \mathbb{R}^{L \times L}$ and bias vector $\mathbf{b} \in \mathbb{R}^L$ depend on the subformulas of $\varphi$ as follows:  
\emph{(i)}~if $\varphi_{\ell}$ is a proposition, $C_{\ell\ell} = 1$, \emph{(ii)}~if $\varphi_{\ell} = \varphi_j \wedge \varphi_k$, then $C_{j \ell} = C_{k \ell} =1$ and $b_{\ell} = -1$; \emph{(iii)}~if $\varphi_{\ell} =\neg \varphi_k$, then $C_{k\ell} = -1$ and $b_{\ell} = 1$, 
and \emph{(iv)}~if  $\varphi_{\ell} = \Diamond_c \varphi_{k}$, then $A_{k\ell}=1$ and $b_{\ell} = -c + 1$. 
All other entries are zero.
Note that if $\varphi \in \ML$,  then  $\N_{\varphi} \in   \GNN{\s}{\ac}$, as required.
\end{proof}

We next show that every classifier in $\GNN{b}{\ac}$ admits an equivalent $\GML{}$ classifier
whereas 
each $\GNN{s}{\ac}$ classifier admits an equivalent $\ML{}$ classifier.
To this end, we first discuss the game-theoretic characterisations of logical 
indistinguishability for $\GML$ and $\ML$.

The \emph{$\ell$-round $c$-graded bisimulation game} \cite{DBLP:journals/corr/abs-1910-00039} is played by Spoiler (him) and Duplicator (her) on  finite pointed models 
$(\mathfrak{M},v)$ and $(\mathfrak{M}',v')$. 
A \emph{configuration} is a tuple $(\mathfrak{M}, w, \mathfrak{M}', w')$, stating that one pebble  is placed on world $w$ in $\mathfrak{M}$, and the other  on $w'$ in $\mathfrak{M'}$.
The initial configuration is $(\mathfrak{M}, v, \mathfrak{M}', v')$.
Each round proceeds in the following two steps, leading to the next configuration. 
(1)~Spoiler selects a pebble and a 
set $U_1 \neq \emptyset$ of at most $c$ neighbours of the world marked by this pebble.
Duplicator responds with a set $U_2$ of neighbours 
of the world marked by the other pebble, such that $|U_1| = |U_2|$.
(2)~Spoiler selects a world in $U_2$  and Duplicator responds with a  world in $U_1$.
If a player cannot pick an appropriate set ($U_1$ or $U_2$), they lose.
If in some configuration worlds marked by pebbles do not satisfy the same propositions, Duplicator loses.
Hence, Duplicator wins  if she has responses for all $\ell$ moves of Spoiler, or if Spoiler cannot make a move in some round.
Spoiler wins if Duplicator loses. 
We write
$(\mathfrak{M},v) \sim_{\ell,c} (\mathfrak{M}', v')$
if Duplicator has a winning strategy starting from configuration $(\mathfrak{M},v,\mathfrak{M}',v')$.
Importantly,
 $\sim_{\ell,c}$ determines indistinguishability of pointed models in $\GML_{\ell,c}$ and any class of pointed models closed under $\sim_{\ell,c}$ is definable by a $\GML_{\ell,c}$ formula.

\begin{restatable}{theorem}{gamegml}\label{th:game-GML}
\cite{DBLP:journals/corr/abs-1910-00039}
For any pointed models  $(\mathfrak{M}, v)$ and $(\mathfrak{M'}, v')$,  and any $\ell , c \in \mathbb{N}$: 
$(\mathfrak{M}, v) \sim_{\ell,c} (\mathfrak{M'}, v')$ iff  $(\mathfrak{M}, v) \equiv_{\GML_{\ell,c}} (\mathfrak{M'}, v')$ 
iff  
$(\mathfrak{M}',v') \models \varphi_{[\mathfrak{M},v]}^{\ell,c}$.

Here, the \emph{characteristic formula} $\varphi_{[\mathfrak{M},v]}^{\ell,c}$ is defined inductively on $n \geq 0$ as follows, for $p \in \prop$ and $U^n_{v,w}$ the set of worlds $u$ such that $\{v,u\} \in E$ and $(\mathfrak{M}, u) \sim_{n,c} (\mathfrak{M}, w)$.
\begin{align*}
& \varphi_{[\mathfrak{M},v]}^{0,c} := \bigwedge \{ p  : (\mathfrak{M},v) \models p \} 
\land
\bigwedge \{ \neg p  : (\mathfrak{M},v) \not\models p \},
\\
& \varphi_{[\mathfrak{M},v]}^{n+1,c} :=  \; \varphi_{[\mathfrak{M},v]}^{0,c} \wedge {}  
\\
&\bigwedge \big \{ \Diamond_k \varphi_{[\mathfrak{M},w]}^{n,c} : \{v,w\} \in E,
 k \leq \mathit{min}( |U_{v,w}^n|, c) \big \} \land {}
\\
&\bigwedge \big \{ \neg \Diamond_k \varphi_{[\mathfrak{M},w]}^{n,c} : \{v,w\} \in E \text{ and }
|U_{v,w}^n|
< k \leq c \big \} \land {}
\\
& 
\neg \Diamond_1 \bigwedge_{\{v,w\} \in E} \neg \varphi_{[\mathfrak{M}, w]}^{n,c}.
\end{align*}
Moreover, any $ \sim_{\ell,c}$-closed class  $\mathcal{C}$ of pointed models is definable  by the formula 
$\bigvee_{(\mathfrak{M},v) \in \mathcal{C}} \varphi_{[\mathfrak{M},v]}^{\ell,c}$.
\end{restatable}

We note two important observations regarding Theorem~\ref{th:game-GML}. First, our construction of 
characteristic formulas corrects an error of \citet{DBLP:journals/corr/abs-1910-00039}, by including a
final conjunct that is necessary for the theorem to hold. Second, we observe that the $\ell$-round $1$-graded bisimulation games coincide with $\ell$-round bisimulation games for $\ML{}$ \cite{DBLP:books/el/07/GorankoO07}, and that the characteristic formulas $\varphi_{[\mathfrak{M},v]}^{\ell,1}$ are characteristic formulas of $\ML$.

We can now shift our attention to GNNs and show that classifiers in  $\GNN{\bm}{\ac}$ and $\GNN{\s}{\ac}$  are invariant under the   bisimulation games  for  $\GML$ and $\ML$,  respectively. 

\begin{restatable}{theorem}{biac}\label{thm:bisimulation-invariance}
The following hold:
\begin{enumerate}
\item $\GNN{\bm}{\ac}$ classifiers with $L$ layers and $k$-bounded aggregation are  $\sim_{L,k}$-invariant;

\item $\GNN{\s}{\ac}$ classifiers with $L$ layers are $\sim_{L,1}$-invariant.
\end{enumerate}
\end{restatable}

\begin{proof}[Proof Sketch]
We show by induction on $\ell  \leq L$ that, for any  
pointed graphs satisfying $(G_1, v_1) \sim_{\ell,k} (G_2,v_2)$, the execution of a $k$-bounded GNN  $\N \in  \GNN{\bm}{\ac}$  satisfies $\emb_1(v_1)^{(\ell)} = \emb_2(v_2)^{(\ell)}$, and thus $\N(G_1,v_1) = \N(G_2,v_2)$.

If $\ell = 0$, $(G_1, v_1) \sim_{0,k} (G_2,v_2)$ implies that
 $v_1$ and $v_2$ satisfy the same propositions, so $\emb_1(v_1)^{(0)} = \emb_2(v_2)^{(0)}$.
For $\ell \geq 1$, $(G_1, v_1) \sim_{\ell,k} (G_2,v_2)$ implies   $(G_1, v_1) \sim_{\ell-1,k} (G_2,v_2)$, so  $\emb_1(v_1)^{(\ell-1)} = \emb_2(v_2)^{(\ell-1)}$ by induction. It remains to show that $\agg(\lBrace  \emb_1(w)^{(\ell-1)} \rBrace_{w \in N_{\G_1}(v_1)})$ equals $\agg( \lBrace  \emb_2(w)^{(\ell-1)} \rBrace_{w \in N_{\G_2}(v_2)})$,  which subsequently implies  $\emb_1(v_1)^{(\ell)} = \emb_2(v_2)^{(\ell)}$ by Equation \eqref{eq:ac}. 
Suppose for the sake of contradiction, and without loss of generality, that there is  a neighbour $w_2$ of $v_2$ such that $\emb_2(w_2)^{(\ell-1)}$ occurs $k_2 < k$ times in $\lBrace  \emb_2(w)^{(\ell-1)} \rBrace_{w \in N_{G_2}(v_2)}$
and 
$k_1 > k_2$  times
in $ \lBrace  \emb_1^{(\ell-1)}(w) \rBrace_{w \in N_{\G_1}(v_1)}$. The strategy for Spoiler is to select a set $U_1$ of $\mathit{min}(k,k_1)$ elements of $w \in N_{G_1}(v_1)$ satisfying $\lambda_1(w)^{(\ell-1)} = \lambda_1(w_1)^{(\ell-1)}$. Duplicator must respond with a subset $U_2$ of neighbours of $v_2$ in $G_2$ of the same cardinality. Any such $U_2$ must contain $w_2$ such that $\lambda_2(w_2)^{(\ell-1)} \neq \lambda_1(w_1)^{(\ell-1)}$. In the second part of the round, Spoiler chooses $w_2$; then, whichever element $w_1'$ in $U_1$ Duplicator chooses, we have $\lambda_1(w_1')^{(\ell-1)} \neq \lambda_2^{(\ell-1)}(w_2)$. Hence, by the inductive hypothesis, $(G_1,w_1') \not\sim_{\ell-1,k} (G_2,w_2)$ and hence 
$(G_1,v_1) \not\sim_{\ell,k} (G_2,v_2)$, raising a contradiction.
The proof for $\N \in \GNN{\s}{\ac}$ is  a particular case, where $k=1$. 
\end{proof}

By \Cref{th:game-GML},  invariance under bisimulation games implies that graphs accepted by a GNN can be characterised by a disjunction of  characteristic formulas.

\begin{corollary}\label{cor:modal-equivalence}
Let $\N$ be a GNN with $L$ layers and  $\mathcal{C}$  the pointed models $(\mathfrak{M},v)$ accepted by $\N$. 
If $\N \in \GNN{\bm}{\ac}$ with $k$-bounded aggregations, 
it is equivalent to the $\GML$ formula $\bigvee_{(\mathfrak{M},v) \in \mathcal{C}} \varphi^{L,k}_{[\mathfrak{M},v]}$. 
If $\N \in \GNN{\s}{\ac}$, 
it is equivalent to the $\ML$ formula $\bigvee_{(\mathfrak{M},v) \in \mathcal{C}} \varphi^{L,1}_{[\mathfrak{M},v]}$. 
\end{corollary}

Therefore, $\GML \geq \GNN{\bm}{\ac}$ and  $\ML \geq  \GNN{\s}{\ac}$.
By combining these results with \Cref{prop:ml-to-gnn} we obtain the following  exact correspondence.

\begin{corollary}\label{cor:modal}
$\GML \equiv \GNN{\bm}{\ac}$ and $\ML \equiv  \GNN{\s}{\ac}$.
\end{corollary}

\section{Modal Logics with Global Counting}\label{sec:readouts}

We now consider bounded GNNs with readouts. 
We first show that \GMLC{}
is captured by $\GNN{\bm}{\acr}$ using max-$k$-sum as aggregation, whereas  
 \MLA{} is captured by $\GNN{\s}{\acr}$ using  componentwise maximum aggregation.

\begin{restatable}{theorem}{linalgacr}\label{prop:ml-to-gnn-acr}
$\GMLC \! \leq \!  \GNN{\bm}{\acr}$ and $\MLA \! \leq \!  \GNN{\s}{\acr}$.
\end{restatable}

\begin{proof}[Proof Sketch]
Let
$\varphi \in \GMLC$; we will construct a GNN by modifying the construction from the proof of \Cref{prop:ml-to-gnn}. 
We now use $\acr$ layers with $\comb(\mathbf{x}, \mathbf{y}, \mathbf{z}) = \sigma(\mathbf{x} \mathbf{C}  + \mathbf{y} \mathbf{A} + \mathbf{z} \mathbf{R} + \mathbf{b})$, where $\mathbf{R}$  simulates global modalities.
In particular, for subformulas $\varphi_\ell =\countE_c \varphi_{k}$, we 
set $R_{k\ell} = 1$  and $b_\ell = -c+1$, whereas   
other entries are set to zeros. Moreover, we let $\readout$ (and $\agg$) be the max-$n$-sum function where $n$ is the counting rank of $\varphi$.
The remaining components are as in the proof of \Cref{prop:ml-to-gnn}. 
The construction for $\MLA$ is obtained as a particular case by taking $c = 1$ and noting the max-1-sum function
corresponds to  max aggregation.
\end{proof}

We next show that $\GNN{\bm}{\acr}$  and  $\GNN{\s}{\acr}$  classifiers admit equivalent $\GMLC$ and $\MLA$ classifiers, respectively. As a first step, our approach involves developing a new game-theoretic characterisation of $\GMLC$ that naturally covers $\MLA$ as a special case.

To this end, we introduce \emph{$\ell$-round c-graded  global bisimulation games}, by extending  the games from  \Cref{sec:agg-combine}. In each round, Spoiler can now choose to play either a standard (local) round as before, or a \emph{global round}.  
Each global round proceeds by Spoiler selecting a non-empty set $U_1$ of worlds of size bounded by $c$ in one of the models, and Duplicator subsequently picking a set $U_2$ of worlds of the same size in the other model; Spoiler then places a pebble in a world $u$ in $U_2$ and Duplicator responds by placing a pebble on a world $u'$ in $U_1$, leading to a new configuration $(\mathfrak{M},u,\mathfrak{M}', u')$.  
We write  $(\mathfrak{M}, v) \sim_{\ell,c}^{\countE} (\mathfrak{M'}, v')$ if Duplicator has a winning strategy in the  $\ell$-round $c$-graded global bisimulation game  starting at  $(\mathfrak{M}, v, \mathfrak{M}', v')$.
Similarly as in the case of non-global games (\Cref{th:game-GML}), we can show the following characterisation result.

\begin{restatable}{theorem}{gameGMLC}\label{th:game-GMLC}
For any pointed models  $(\mathfrak{M}, v)$ and $(\mathfrak{M'}, v')$  and $\ell,c \in \mathbb{N}$,  $(\mathfrak{M}, v) \sim_{\ell,c}^{\countE} (\mathfrak{M'}, v')$ iff 
 $(\mathfrak{M},v) \equiv_{\GMLC_{\ell,c}} (\mathfrak{M'}, v')$ iff
$(\mathfrak{M}',v') \models \varphi_{\countE[\mathfrak{M},v]}^{\ell,c}$.

Here, the characteristic formulas are defined inductively on $n \geq 0$ as follow, where  
$U_{v,w}^n$ is the set of all $u$ with $\{v,u\} \in E$ and $(\mathfrak{M}, u) \sim_{n,c}^{\countE} (\mathfrak{M}, w)$, and
$J_w^n$ is the set of all $u \in  V$ with $(\mathfrak{M}, u) \sim_{n,c}^{\countE} (\mathfrak{M}, w)$.
\begin{align*}
& \varphi_{\countE[\mathfrak{M},v]}^{0,c} := \bigwedge \{ p  : (\mathfrak{M},v) \models p \} 
\land
\bigwedge \{ \neg p  : (\mathfrak{M},v) \not\models p \},
\\
& \varphi_{\countE[\mathfrak{M},v]}^{n+1,c} := \; \varphi_{\countE[\mathfrak{M},v]}^{0,c} \wedge {}  
\\
&\bigwedge \big \{ \Diamond_k \varphi_{\countE[\mathfrak{M},w]}^{n,c} : \{v,w\} \in E,
 k \leq \mathit{min}( |U_{v,w}^n|, c) \big \}  \land {}
\\
&\bigwedge \big \{ \neg \Diamond_k \varphi_{\countE[\mathfrak{M},w]}^{n,c} : \{v,w\} \in E \text{ and }
|U_{v,w}^n|
< k \leq c \big \} \land {}
\\
&\bigwedge \big \{ \countE_{k} \varphi_{\countE[\mathfrak{M},w]}^{n,c} : w \in V,
 k \leq 
|J_w^n|, \text{ and } k \leq c \big \} \land {}
\\
&\bigwedge \{ \neg \countE_{k} \varphi_{\countE[\mathfrak{M},w]}^{n,c} : w \in V \text{ and }
|J_w^n|
< k \leq c \big \} \land{}
\\
& \neg \Diamond_1 \bigwedge_{\{v,w\} \in E} \neg \varphi_{\countE[\mathfrak{M},w]}^{n,c} \land{}
\neg \countE_1 \bigwedge_{w \in V} \neg \varphi_{\countE[\mathfrak{M},w]}^{n,c}.
\end{align*}

Moreover, any class $\mathcal{C}$ of pointed models closed under $\sim_{\ell,c}^{\countE}$  is definable  by $\bigvee_{(\mathfrak{M},v) \in \mathcal{C}} \varphi_{\countE[\mathfrak{M},v]}^{\ell,c}$.    
\end{restatable}

We can observe that if $c=1$, then characteristic formulas are in \MLE{}, and our games correspond to global bisimulation games developed for \MLE{} \cite[Section~5.1]{DBLP:books/el/07/GorankoO07}. 


We are now ready to show that classifiers in $\GNN{\bm}{\acr}$  and  $\GNN{\s}{\acr}$   are invariant under the bisimulation games for \GMLC{}  and  \MLE{}, respectively.

\begin{restatable}{theorem}{acrbi}\label{thm:arc-bisimulation-invariance}
The following hold:
\begin{enumerate}
     \item $\GNN{\bm}{\acr}$ classifiers with $L$ layers and k-bounded aggregations and readouts are $\sim_{L,k}^{\countE}$-invariant;
    \item $\GNN{\s}{\acr}$ classifiers with $L$ layers are $\sim_{L,1}^{\countE} $-invariant.  
\end{enumerate}
\end{restatable}

\begin{proof}[Proof Sketch]
The proof has the same structure as in Theorem \ref{thm:bisimulation-invariance}.
The inductive hypothesis implies  
$\emb_1(v_1)^{(\ell-1)} = \emb_2(v_2)^{(\ell-1)}$ and
$\mathit{min}(k,\lBrace{\emb_1(w)^{(\ell-1)}\rBrace}_{w \in N_{\G_1}(v_1) })$ equals $\mathit{min}(k,\lBrace{ \emb_2(w)^{(\ell-1)}\rBrace}_{w \in N_{\G_2}(v_2) })$. Additionally we can show now that   $\readout( \lBrace  \emb_1(w)^{(\ell-1)} \rBrace_{w \in V_1})$ equals $\readout( \lBrace  \emb_2(w)^{(\ell-1)} \rBrace_{w \in V_2})$, which then implies that $\emb_1(v_1)^{(\ell)} = \emb_2(v_2)^{(\ell)}$ by Equation \eqref{eq:integrated}. Indeed, if these multisets were not equal, 
Spoiler could find (w.l.o.g.) a node \ $w_2 \in V_2$ such that $\emb_2(w_2)^{(\ell-1)}$ occurs $k_2 < k$ times in $\lBrace  \emb_2(w)^{(\ell-1)} \rBrace_{w \in V_2}$
and 
$k_1 > k_2$  times
in $ \lBrace  \emb_1^{(\ell-1)}(w) \rBrace_{w \in V_1}$. This would allow Spoiler to play a global round that wins the game. 

The proof for the set-based case can again be obtained as a particular case, with $c=1$. 
\end{proof}

As before, Theorems \ref{th:game-GMLC} and \ref{thm:arc-bisimulation-invariance} imply the following.

\begin{corollary}\label{cor:global-modal-equivalence}
Let $\N$ be a GNN with $L$ layers and  $\mathcal{C}$  the pointed models $(\mathfrak{M},v)$ accepted by $\N$. 
If $\N \in \GNN{\bm}{\acr}$ with $k$-bounded aggregations, 
it is equivalent to the $\GMLC$ formula $\bigvee_{(\mathfrak{M},v) \in \mathcal{C}} \varphi^{L,k}_{\countE[\mathfrak{M},v]}$. 
If $\N \in \GNN{\s}{\acr}$, 
it is equivalent to the $\MLA$ formula $\bigvee_{(\mathfrak{M},v) \in \mathcal{C}} \varphi^{L,1}_{\countE[\mathfrak{M},v]}$. 
\end{corollary}

Therefore, $\GMLC \geq \GNN{\bm}{\acr}$ and $\MLA \geq  \GNN{\s}{\acr}$.
Combining these with \Cref{prop:ml-to-gnn-acr} we have the following.

\begin{corollary}\label{cor:acrb}
The following equivalences hold: $\GMLC \equiv \GNN{\bm}{\acr}$ and $\MLA \equiv \GNN{\s}{\acr}$.
\end{corollary}

\section{Two-Variable Fragments} \label{sec:two-var}
Finally, we consider bounded GNNs with non-neighbour aggregation. We show that $\GNN{\bm}{\acn}$  and  $\GNN{\s}{\acn}$  exactly correspond to $\Ctwo$ and $\FOtwo$, respectively. 

We first show that each 
$\Ctwo{}$ classifier can be expressed in $\GNN{\bm}{\acn}$ using max-$k$-sum for neighbour and non-neighbour aggregations, whereas each
$\FOtwo{}$ classifier can be expressed in $\GNN{\s}{\acn}$ using componentwise max.

\begin{restatable}{theorem}{linalgacn}\label{lem:twovar-to-gnn}
$\Ctwo{} \leq \GNN{\bm}{\acn}$ and $\FOtwo \leq  \GNN{\s}{\acn}$.
\end{restatable}

\begin{proof}[Proof Sketch]
Similarly to \citet{DBLP:conf/iclr/BarceloKM0RS20}, we exploit the fact that  \Ctwo{} has the same expressive power as the modal logic \EMLC{} with complex modalities \cite[Theorem 1]{DBLP:conf/csl/LutzSW01}.
We can show that \EMLC{} classifiers in normal form \cite[Lemma D.4]{DBLP:conf/iclr/BarceloKM0RS20} can be captured by $\GNN{\bm}{\acn}$.
The construction is similar to  that of \Cref{prop:ml-to-gnn}, but using \acn{} layers.
We use max-$n$-sum aggregations where $n$ is the counting rank of the \EMLC{} formula.
Then, $\comb(\mathbf{x}, \mathbf{y}, \mathbf{z}) = \sigma( \mathbf{x}  \mathbf{C} + \mathbf{y} \mathbf{A} + \mathbf{z} \overline{\mathbf{A}} + \mathbf{b})$, with matrix and vector entries  depending on the subformulas  $\varphi_\ell$. 
If $\varphi_{\ell}$ is a proposition, conjunction, or negation,  the $\ell$th columns of $\mathbf{A}$, $\mathbf{C}$, and  $\mathbf{b}$ are as in  \Cref{prop:ml-to-gnn}, and the $\ell$th column of  $\overline{\mathbf{A}}$ has only $0$s. For the remaining cases,
the $\ell$th columns of $\mathbf{A}$, $\mathbf{C}$, $\overline{\mathbf{A}}$, and  $\mathbf{b}$ are defined as in the construction of  \cite[Theorem 5.1]{DBLP:conf/iclr/BarceloKM0RS20}, except that we use combinations of bounded neighbour and non-neighbour aggregation instead of combinations of unbounded aggregation and global readouts.
The proof for \FOtwo{} is a particular case, where operators in $\EMLC$ can count up to 1.
\end{proof}

Games for two-variable logics \cite{DBLP:books/sp/Libkin04} are 
similar to bisimulation games, but they are now played with
two pairs of pebbles.
In what follows we define a variant of the game for \Ctwo{} \cite{DBLP:journals/tcs/GradelO99}, obtained by imposing restrictions on both the number of rounds and on the possible counting.
We let the \emph{2-pebble $\ell$-round $c$-graded game} be  played on two models $\mathfrak{M}$ and $\mathfrak{M'}$
by Spoiler and Duplicator  with two pairs of pebbles:
$(p_{\mathfrak{M}}^1, p_{\mathfrak{M}'}^1)$ and  $(p_{\mathfrak{M}}^2, p_{\mathfrak{M}'}^2)$. 
After each round, the pebble positions  define a mapping $\pi$ of two elements in $\M$ into two elements of $\M'$. 
Duplicator has a winning strategy if she can ensure that, after each round, 
$\pi$ is a partial isomorphism between the models.
For node classification, we consider games in which the starting configuration has $p_{\mathfrak{M}}^1$ and $p_{\mathfrak{M}'}^1$ placed on some elements of $\M$ and $\M'$, respectively.
Each round is played as follows: (1) Spoiler chooses a model (say, $\mathfrak{M}$), one of pebble pairs $i \in \{1,2\}$, and  a non-empty
subset $U \subseteq V$ with  $|U| \leq c$.
Duplicator responds with a subset $U' \subseteq V'$ such that $|U'| = |U|$.
(2) Spoiler places pebble  
$p_{\mathfrak{M}'}^i$ on some $u' \in U'$. Duplicator responds by placing  
$p_{\mathfrak{M}}^i$ on some $u \in U$.
We write  $(\mathfrak{M}, a) \sim^{2}_{\ell,c}  (\mathfrak{M}',a')$ if Duplicator has a winning strategy when  $p_{\mathfrak{M}}^1$ and $p_{\mathfrak{M}'}^1$ are initially placed on elements $a$ and $a'$, respectively.

As we establish next, this game variant  characterises indistinguishability in the logic $\Ctwo{}_{\ell,c}$.
What is crucial is that we consider both bounded depth and counting rank.
As a result, formulas of $\Ctwo{}_{\ell,c}$ have finitely many equivalence classes (with respect to the logical equivalence), and so, any class of models closed under the equivalence in $\Ctwo{}_{\ell,c}$ is definable by a (finite) $\Ctwo{}_{\ell,c}$ formula.



\begin{restatable}{theorem}{ctwogame}\label{th:Ctwo-game}
For any pointed models  $(\mathfrak{M}, a)$ and $(\mathfrak{M'}, a')$ and any $\ell,c \in \mathbb{N}$:
$(\mathfrak{M}, a) \sim^{2}_{\ell,c} (\mathfrak{M}',a')$ iff $a$ in $\mathfrak{M}$ and $a'$ in $\mathfrak{M}'$ satisfy the same $\Ctwo_{\ell,c}$ formulas with one free variable. 
Furthermore, any class $\mathcal{C}$ of pointed models closed under $\sim^{2}_{\ell,c}$ is definable by a $\Ctwo_{\ell,c}$ formula.
\end{restatable}
\begin{proof}[Proof Sketch]
To show the equivalence, we prove a stronger result, where  $a$ and $a'$ are vectors of length at most  2.  
We show each implication by induction on $\ell$.
For the forward implication,
we show the contrapositive, namely
that $\mathfrak{M} \models \varphi(a)$ and  $\mathfrak{M}' \not\models \varphi(a')$, imply the existence of a winning strategy for Spoiler. 
For the opposite direction,
we show that whenever
$\mathfrak{M} \models \varphi(\mathbf{a})$ iff $\mathfrak{M}' \models \varphi(\mathbf{a}')$,
there is a winning strategy for Duplicator.

The above shows that $(\mathfrak{M}, a) \sim^{2}_{\ell,c} (\mathfrak{M}',a')$ iff $a$ in $\mathfrak{M}$ and $a'$ in $\mathfrak{M}'$ satisfy the same $\Ctwo_{\ell,c}$ formulas with one free variable.
Since, up to logical equivalence,  there are finitely many $\Ctwo_{\ell,c}$ formulas  \cite[Lemma 4.4]{cai1992optimal}, each $\sim_{\ell,c}^{2}$ equivalence class can be  expressed as a (finite) disjunction of (finite) $\Ctwo_{\ell,c}$ formulas.
\end{proof}

Next, we show that bounded GNNs with non-neighbour aggregation are invariant under our variant of the 2-pebble games.

\begin{restatable}{theorem}{biacn}\label{lem:pebble-invariance}
The following hold:
\begin{enumerate}
\item $\GNN{\bm}{\acn}$ classifiers with $L$ layers and $k$-bounded aggregations and readout are  $\sim^{2}_{L,k}$-invariant;
\item $\GNN{\s}{\acn}$ classifiers with $L$ layers are $\sim^{2}_{L,1}$-invariant.
\end{enumerate}
\end{restatable}

\begin{proof}[Proof Sketch]
The structure of the proof follows that of  \Cref{thm:bisimulation-invariance}, but now  games use two pairs of pebbles and GNNs have  $\acn$ layers---with two types of aggregation.
The important part of the proof is in the  inductive step, where we show that $(G_1,v_1) \sim_{\ell,k}^{2} (G_2, v_2)$ implies that
$\mathit{min}(k,\lBrace{\emb_1(w)^{(\ell-1)}\rBrace}_{w \in X_{\G_1}(v_1) })$  equals $\mathit{min}(k,\lBrace{ \emb_2(w)^{(\ell-1)}\rBrace}_{w \in X_{\G_2}(v_2) })$, for both $X \in \{\N ,\Nn\}$.
Towards a contradiction suppose that there is $u \in X_{\G_1}(v_1)$ such that $\emb_1(u)^{(\ell-1)}$ appears $k_1$ times in $\lBrace   \emb_1(w)^{(\ell-1)}  \rBrace _{w \in X_{\G_1}(v_1)}$ and $k_2 < k_1 $ times in $\lBrace   \emb_2(w)^{(\ell-1)}  \rBrace _{w \in X_{\G_2}(v_2)}$, with $k_2 < k$.
The winning strategy for Spoiler is to  pick a set $U_1$ of $\min(k,k_1)$ elements from $\lBrace   \emb_1(w)^{(\ell-1)}  \rBrace _{w \in X_{\G_1}(v_1)}$ with label $\emb_1(u)^{(\ell-1)}$. 
%
This allows Spoiler to get to a configuration $(\mathfrak{M}_G,v_1,u_1,\mathfrak{M}_{G'},v_2,u_2)$ with $\emb_1(u)^{(\ell-1)} \neq \emb_2(u_2)^{(\ell-1)}$ which, by the inductive hypothesis, implies
$(G_1, v_1) \not\sim_{\ell,k}^{2} (G_2, v_2)$.
%
%
The proof for the set-based case is a particular case, with $k=1$. 
%
\end{proof}


As before,  \Cref{th:Ctwo-game} implies the following.

\begin{corollary}\label{cor:FOtwo-Ctwo-equivalence}
Let $\N$ be a GNN with $L$ layers and  $\mathcal{C}$  the pointed models $(\mathfrak{M},v)$ accepted by $\N$. 
If $\N \in \GNN{\bm}{\acn}$ 
it is equivalent to a $\Ctwo$ formula. 
If $\N \in \GNN{\s}{\acn}$, 
it is equivalent to an $\FOtwo$ formula. 
\end{corollary}

Therefore, $\Ctwo \geq \GNN{\bm}{\acn}$ and $\FOtwo \geq  \GNN{\s}{\acn}$.
Combining these with \Cref{lem:twovar-to-gnn} we have the following.

\begin{corollary}\label{cor:acn}
The following equivalences hold: $\Ctwo \equiv \GNN{\bm}{\acn}$ and $\FOtwo \equiv \GNN{\s}{\acn}$.
\end{corollary}

\section{Conclusion and Future Work}\label{sec:conclusions}


We have introduced families of bounded GNNs, whose expressive power corresponds exactly to well-known modal logics and 2-variable first-order logics. 
Among others, we have showed that standard aggregate-combine GNNs with bounded aggregation have the same expressive power as the graded modal logic.
The correspondence between \FO{}-expressibility and bounding aggregation (and readout) occurs as an interesting phenomenon to study. 
In particular, we find it interesting to  determine for which classes of GNNs classifiers, \FO-expressibility is equivalent to expressibility by bounded GNNs.
Future work directions we consider include also establishing tight bounds on the size of logical formulas capturing  GNNs and practical extraction of logical formulas from GNNs.

\section*{Acknowledgements}
Eva Feng is generously supported by a Google DeepMind Scholarship (CS2324\_DeepMind\_1594092). 

\bibliography{aaai2026}


\newpage
\onecolumn
\appendix

\section*{Technical Appendix}






\section*{Proof Details for \Cref{boundedGNN}}

\finitesp*
\begin{proof}
The fact that $|\mathsf{sp}(\N,0)| = 2^d$ follows from the definition of spectra.
Next, we will compute the bound on $\emb(v)^{(\ell+1)}$.
If $\N$ uses \acr{} layers, then $$\emb(v)^{(\ell+1)} = \comb \Big( \emb(v)^{(\ell)}, 
\agg( \lBrace  \emb(w)^{(\ell)} \rBrace_{w \in N_G(v)} ), \readout( \lBrace  \emb(w)^{(\ell)} \rBrace_{w \in V}) \Big). $$ 
We will analyse the number of values that each of the three components the argument of $\comb$  can take.
Multiplication of these  numbers will provide us with a bound for
 $| \mathsf{sp}(\N,\ell+1)|$.
 \begin{itemize}
 \item
The first component, $\emb(v)^{(\ell)}$, can take up to  $| \mathsf{sp}(\N,\ell)|$ values.
\item 
Now, consider the second component, $\agg( \lBrace  \emb(w)^{(\ell)} \rBrace_{w \in N_G(v)} )$.
Since aggregation is
$k$-bounded, 
the output of aggregation  can  take at most as many values as there are multisets consisting of  elements in $\mathsf{sp}(\N, \ell)$ with at most $k$ multiples for each element.
That is, the aggregation can take at most $(k+1)^{|\mathsf{sp}(\N, \ell)|}$ values.

\item Calculations for the third component are analogous as for the second component, and so, the third component can also take up to $(k+1)^{|\mathsf{sp}(\N, \ell)|}$ values.
\end{itemize}
Hence
$$|\mathsf{sp}(\N, \ell+1)|
=
|\mathsf{sp}(\N, \ell)| \cdot
(k+1)^{|\mathsf{sp}(\N, \ell)|} \cdot
(k+1)^{|\mathsf{sp}(\N, \ell)|}
=
|\mathsf{sp}(\N, \ell)| \cdot
(k+1)^{2 \cdot |\mathsf{sp}(\N, \ell)|}.
$$
If $\N$ uses $\acn$ layers, we obtain the same bounds, since the third component is replaced with a copy of the second component (for the second aggregation function). 
If $\N$ uses $\ac$ layers only, then the third component is not present, and so
$$|\mathsf{sp}(\N, \ell+1)|
=
|\mathsf{sp}(\N, \ell)| \cdot
(k+1)^{|\mathsf{sp}(\N, \ell)|}.
$$
\end{proof}



\newpage

\section*{Proof Details for Section \ref{sec:agg-combine}}

\linalgac*
\begin{proof}
Let $\varphi \in \GML$ be a logic classifier of dimension $d$ with propositions $\prop(\varphi)$ and subformulas $\mathsf{sub}(\varphi) = (\varphi_1, \ldots, \varphi_L)$, where $k \leq \ell$ whenever $\varphi_k$ is a subformula of $\varphi_{\ell}$.
We construct a GNN classifier $\N_{\varphi}$ with layers $0, \ldots, L$.
The classification function $\cls$ is such that $\cls(\mathbf{x}) = \true$ iff the last element of $\mathbf{x}$ is $1$.
Layer $0$  uses a combination function to multiply input vectors by a  matrix $\mathbf{D} \in \R^{d \times L}$, namely  $\lambda(v)^{(1)} =  \lambda(v)\mathbf{D} $, where $\mathbf{D}_{k \ell} = 1$ if the $k$th position of the  input vectors  corresponds to a proposition $\varphi_\ell$;
other entries of $\mathbf{D}$ are $0$.
Other layers are of dimension $L$ and defined below. They are homogeneous, meaning that the aggregation, combination and readout functions in each of these $L$ layer are the same. 

The activation function is the componentwise truncated ReLU defined as $\sigma(x) = \mathit{min}(\mathit{max}(0,x),1)$. 
The combination function is $\comb(\mathbf{x}, \mathbf{y}) = \sigma(\mathbf{x} \mathbf{C}  + \mathbf{y} \mathbf{A} + \mathbf{b})$ and the aggregation function $\agg$ is the max-$n$-sum function, where $n$ is the counting rank of $\varphi$. 
The entries of $\mathbf{A}, \mathbf{C} \in \mathbb{R}^{L \times L}$ and  $\mathbf{b} \in \mathbb{R}^L$ depend on the subformulas of $\varphi$ as follows: 

\begin{center}
\begin{tabular}{l}
1. if $\varphi_{\ell}$ is a proposition, $C_{\ell\ell} = 1$, \\
2. if $\varphi_{\ell} = \varphi_j \wedge \varphi_k$, then $C_{j \ell} = C_{k \ell} =1$ and $b_{\ell} = -1$; \\
3. if $\varphi_{\ell} =\neg \varphi_k$, then $C_{k\ell} = -1$ and $b_{\ell} = 1$;\\
4. if $\varphi_{\ell} = \Diamond_c \varphi_k$, then $A_{k\ell} = 1$ and $b_{\ell} = -c + 1$; and\\
\end{tabular}
\end{center}

\noindent all other entries  are set to $0$.

Consider the GNN application to $\G = (V,E,\emb)$ of dimension $d$.
We  show that for each subformula $\varphi_{\ell}$, each  $i \in \{\ell, \ldots, L\}$, and  $v \in V$, if   $(\M_{\G},v) \models \varphi_{\ell}$ then $\emb(v)^{(i)}_{\ell} =1 $, and  otherwise $\emb(v)^{(i)}_{\ell} =0$. 
This implies that $\N_\varphi(\G,v) = \true$ iff $(\M_{\G},v) \models \varphi$.
The proof is by induction on the structure of $\varphi_\ell$.

\begin{itemize}
\item If $\varphi_{\ell}$ is a proposition then, by the design of layer $0$, we have  $\lambda(v)^{(0)}_\ell =1$ if $(\M_{G},v) \models \varphi_\ell$, and otherwise $\lambda(v)^{(0)}_\ell =0$.
Moreover, since $\varphi_\ell$ is a proposition,
all \ac{} layers have $C_{\ell \ell} = 1$, $b_{\ell} = 0$, and $A_{k\ell} = 0$ for each $k$, so
$\comb(\mathbf{x}, \mathbf{y})_{\ell} = \mathbf{x}_{\ell}$.
Thus $\lambda(v)^{(0)}_\ell = \lambda(v)^{(i)}_\ell$ for any $i$.

\item If $\varphi_{\ell} = \neg \varphi_k$, then by construction, we have $C_{k\ell} = -1$, $b_\ell = 1$, and $A_{m\ell} = 0$ for each $m$. 
Then by Equation~\eqref{eq:ac}, we have $\emb(v)^{(i)}_\ell = \sigma(-\emb(v)^{(i-1)}_k + 1)$.
By inductive hypothesis, 
$\lambda(v)^{(i-1)}_k = 1$ if  $(\M_{G},v) \models \varphi_k$,  and  otherwise $\lambda(v)^{(i-1)}_k = 0$. 
Hence, 
$\emb(v)^{(i)}_\ell=1$ if $(\M_{\G} , v) \models  \neg \varphi_k$, and otherwise $\emb(v)^{(i)}_\ell=0$ for all $i \geq \ell$.

\item If $\varphi_{\ell} = \varphi_j \wedge \varphi_k$, then by construction, we have $C_{j\ell} = C_{k\ell} = 1$, $b_\ell = -1$, and $A_{m\ell} = 0$ for each $m$.
Then by Equation~\eqref{eq:ac}, we have $\emb(v)^{(i)}_\ell = \sigma(\emb(v)^{(i-1)}_j + \emb(v)^{(i-1)}_k - 1)$.
By inductive hypothesis, 
$\lambda(v)^{(i-1)}_k = 1$ if  $(\M_{G},v) \models \varphi_k$,  and  otherwise $\lambda(v)^{(i-1)}_k = 0$; and $\lambda(v)^{(i-1)}_j = 1$ if  $(\M_{G},v) \models \varphi_j$, and  otherwise $\lambda(v)^{(i-1)}_j = 0$.
Hence, 
$\emb(v)^{(i)}_\ell=1$ if $(\M_{\G} , v) \models  \varphi_j \wedge \varphi_k$, and otherwise $\emb(v)^{(i)}_\ell=0$ for all $i \geq \ell$.

\item If $\varphi_{\ell} = \Diamond_c \varphi_k$, then by construction, we have $A_{k \ell} = 1$,  $b_{\ell} = -c+1$, and  $C_{m \ell} = 0$ for each $m$. 
Since \agg{} is the max-$n$-sum, by Equation~\eqref{eq:ac}, we have $\emb(v)^{(i)}_\ell = \sigma(\mathit{min}(n, \sum_{w\in N_G(v)} \lambda(w)^{(i-1)}_k) -c + 1)$.
By inductive hypothesis, 
$\lambda(w)^{(i-1)}_k = 1$ if  $(\M_{G},w) \models \varphi_k$,  and  otherwise $\lambda(w)^{(i-1)}_k = 0$. 
Hence, 
$\emb(v)^{(i)}_\ell=1$ if $(\M_{\G} , v) \models \Diamond_c \varphi_k$, and otherwise $\emb(v)^{(i)}_\ell=0$.
\end{itemize}

Finally, observe that $\ML$ is a fragment of $\GML$ where the counting rank is restricted to $1$. 
Then, the max-$n$-sum function becomes the max-$1$-sum, which is simply the component-wise maximum function.
Then, it is easy to see that we have $\ML \leq \GNN{\s}{\ac}$ when the counting rank is restricted to 1.
\end{proof}

\begin{example*}[$\GNN{\bm}{\ac}$ Construction for $\GML$]
In what follows we will show how the construction from \Cref{prop:ml-to-gnn} works on the $\GML_{\ell,c}$ formula $\varphi = \Diamond_3 p \land \lnot \Diamond_3 q$, with quantifier rank $\ell = 1$ and counting rank $c = 3$.
We  assume that input graphs have labellings of the form $[x_p, x_q]$, corresponding to the node colours $p$ and $q$. The first step is to provide an ordering of subformulas of $\varphi$ as described in the proof of \Cref{prop:ml-to-gnn}, for example: $\varphi_1 = p$, $\varphi_2 = q$, $\varphi_3 = \Diamond_3 p$, $\varphi_4 = \Diamond_3 q$, $\varphi_5 = \lnot \Diamond_3 q$, and $\varphi_6 = \varphi = \Diamond_3 p \land \lnot \Diamond_3 q$. 

Next, we show how the corresponding $\N$ is constructed.
Layer $0$ of $\N$  extends the input label using a matrix $\mathbf{D}$, which in our example is as follows:
$$\mathbf{D} = 
\begin{bmatrix}
1 & 0 & 0 & 0 & 0 & 0\\
0 & 1 & 0 & 0 & 0 & 0
\end{bmatrix}.
$$

The remaining 6 layers of $\N$ use the following matrices and bias vector:
$$\mathbf{C} = 
\begin{bmatrix}
1 & 0 & 0 & 0 & 0 & 0\\
0 & 1 & 0 & 0 & 0 & 0\\
0 & 0 & 0 & 0 & 0 & 1\\
0 & 0 & 0 & 0 & -1 & 0 \\
0 & 0 & 0 & 0 & 0 & 1\\
0 & 0 & 0 & 0 & 0 & 0
\end{bmatrix}
,~~
\mathbf{A} = 
\begin{bmatrix}
0 & 0 & 1 & 0 & 0 & 0\\
0 & 0 & 0 & 1 & 0 & 0\\
0 & 0 & 0 & 0 & 0 & 0\\
0 & 0 & 0 & 0 & 0 & 0 \\
0 & 0 & 0 & 0 & 0 & 0\\
0 & 0 & 0 & 0 & 0 & 0
\end{bmatrix}
, \text{ and }~
\mathbf{b} = 
\begin{bmatrix}
0 & 0 & -2 & -2 & 1 & -1
\end{bmatrix}.
$$
\end{example*}

\gamegml*
\begin{proof}
The proof by \citet{DBLP:journals/corr/abs-1910-00039} is not fully correct, due to an incomplete construction of characteristic formulas. Hence, we will reproduce the proof for our complete construction of characteristic formulas.
To prove the theorem we will show three implications: 

\begin{center}
\begin{tabular}{l}
1. $(\mathfrak{M}, v) \sim_{\ell,c} (\mathfrak{M'}, v')$ implies  $(\mathfrak{M}, v) \equiv_{\GML_{\ell,c}} (\mathfrak{M'}, v')$, \\
2. $(\mathfrak{M}, v) \equiv_{\GML_{\ell,c}} (\mathfrak{M'}, v')$ implies $(\mathfrak{M}',v') \models \varphi_{[\mathfrak{M},v]}^{\ell,c}$, and \\
3. $(\mathfrak{M}',v') \models \varphi_{[\mathfrak{M},v]}^{\ell,c}$ implies $(\mathfrak{M}, v) \sim_{\ell,c} (\mathfrak{M'}, v')$.
\end{tabular}
\end{center}

For the first implication we prove  by induction on $\ell$ that for any fixed $c$,  $(\mathfrak{M},v) \not\equiv_{\GML_{\ell,c}} (\mathfrak{M'}, v')$
implies $(\mathfrak{M}, v) \not\sim_{\ell,c} (\mathfrak{M'}, v')$. 
For the base case $\ell = 0$, we have that $(\mathfrak{M}, v) \sim_{0,c} (\mathfrak{M'}, v')$  implies that $v$ and $v'$ satisfy the same propositions, so $(\mathfrak{M},v) \equiv_{\GML_{0,c}} (\mathfrak{M'}, v')$. 
For the inductive step, assume $(\mathfrak{M},v) \not\equiv_{\GML_{\ell,c}} (\mathfrak{M'}, v')$ implies $(\mathfrak{M}, v) \not\sim_{\ell,c} (\mathfrak{M'}, v')$. 
We show that the claim holds for $\ell+1$. 
Assume $(\mathfrak{M},v) \not\equiv_{\GML_{\ell+1,c}} (\mathfrak{M'}, v')$, then there exists $\varphi \in \GML_{\ell+1,c}$ such that $(\mathfrak{M},v) \models \varphi$ but $(\mathfrak{M'},v') \not\models \varphi$. 
It suffices to show the case for $\varphi = \Diamond_k \phi$, where $k \leq c$ and $\phi \in \GML_{\ell, c}$. 
The  remaining cases regarding conjunctions and negations immediate imply that $(\mathfrak{M},v) \not\equiv_{\GML_{\ell+1,c}} (\mathfrak{M'}, v')$ implies $(\mathfrak{M}, v) \not\sim_{\ell+1,c} (\mathfrak{M'}, v')$ by the inductive hypothesis.
We will show that Spoiler has a winning strategy from $(\mathfrak{M},v,\mathfrak{M'},v')$ by picking a set $U_1$ of $k$ worlds in $N_{\mathfrak{M}}(v)$ such that for each $u \in U_1$, $(\mathfrak{M}, u) \models \phi$. 
Then any subset $U_2 \subseteq N_{\mathfrak{M'}}(v')$ of size $k$ picked  by Duplicator must contain some $u'$ such that $(\mathfrak{M'},u') \not\models \phi$. 
Spoiler has a winning strategy by picking $u'$.
Then, whichever $u \in U_1$ Duplicator chooses, we have
$(\mathfrak{M},u) \models \phi$
and $(\mathfrak{M'},u') \not\models \phi$, so by the inductive hypothesis,  $(\mathfrak{M}, u) \not\sim_{\ell,c} (\mathfrak{M'}, u')$. Hence  $(\mathfrak{M}, v) \not\sim_{\ell+1,c} (\mathfrak{M'}, v')$.

The second implication results from the construction of the characteristic formula $\varphi_{[\mathfrak{M},v]}^{\ell,c}$. We can see that $\varphi_{[\mathfrak{M},v]}^{\ell,c}$ essentially characterises the structure of $(\mathfrak{M},v)$ by capturing the presence of each neighbouring type of the appropriate cardilanities by $\bigwedge \big \{ \Diamond_k \varphi_{[\mathfrak{M},w]}^{\ell,c} : \{v,w\} \in E,
 k \leq \mathit{min}( |U_{v,w}^\ell|, c) \big \}$, prohibiting the absence of neighbouring types by $\neg \Diamond_1 \bigwedge_{\{v,w\} \in E} \neg \varphi_{[\mathfrak{M}, w]}^{\ell,c}$, and prohibiting the absence of each neighbouring type of the appropriate cardinalities through $\bigwedge \big \{ \neg \Diamond_k \varphi_{[\mathfrak{M},w]}^{\ell,c} : \{v,w\} \in E \text{ and }
|U_{v,w}^\ell|
< k \leq c \big \}$. Hence, by construction, we have $(\mathfrak{M},v) \models \varphi_{[\mathfrak{M},v]}^{\ell,c}$, so $(\mathfrak{M},v) \equiv_{\GML_{\ell,c}} (\mathfrak{M'}, v')$ implies $(\mathfrak{M'},v') \models \varphi_{[\mathfrak{M},v]}^{\ell,c}$.

Now, we show the third implication  by induction on $\ell$, for a fixed $c$. 
For the base case when $\ell = 0$, $(\mathfrak{M'},v') \models \varphi_{[\mathfrak{M},v]}^{0,c}$ implies that $v$ and $v'$ satisfy the same propositions, so we have $(\mathfrak{M}, v) \sim_{0,c} (\mathfrak{M'}, v')$. 
For the inductive step, assume $(\mathfrak{M'},v') \models \varphi_{[\mathfrak{M},v]}^{\ell,c}$ implies $(\mathfrak{M}, v) \sim_{\ell,c} (\mathfrak{M'}, v')$.
We will show that  if $(\mathfrak{M'},v') \models \varphi_{[\mathfrak{M},v]}^{\ell+1,c}$, then $(\mathfrak{M}, v) \sim_{\ell+1,c} (\mathfrak{M'}, v')$.
To this end, we will show that  Duplicator has a winning strategy.
First, we need to show that $(\mathfrak{M},v,\mathfrak{M}',v')$ is not a winning configuration for Spoiler, that is, $(\mathfrak{M}, v)$ and $(\mathfrak{M'}, v')$ satisfy the same propositions.
Indeed, by the form of $\varphi_{[\mathfrak{M},v]}^{\ell,c}$, the fact that $(\mathfrak{M'},v') \models \varphi_{[\mathfrak{M},v]}^{\ell+1,c}$ implies  $(\mathfrak{M'},v') \models \varphi_{[\mathfrak{M},v]}^{0,c}$.
The latter, implies that $(\mathfrak{M}, v)$ and $(\mathfrak{M'}, v')$ satisfy the same propositions.
To show how the winning strategy for Duplicatior, we will consider  separately  the cases when Spoiler plays a set in $\mathfrak{M}$ and in $\mathfrak{M}'$. 

Assume that Spoiler selects a set $U = \{u_1,\dots, u_t \}$ of size $t \leq c$ of elements in $N_{\mathfrak{M}}(v)$. 
We will show that there exists a set $U' = \{u_1',\dots, u_t' \}$ of elements in $N_{\mathfrak{M'}}(v')$
such that
$(\mathfrak{M'}, u_i') \models \varphi_{[\mathfrak{M},u_i]}^{\ell,c}$, for each $i \in \{ 1, \dots, t \}$.
To this end, it suffices to show that if there is a set $Z=\{z_1, \dots, z_{t'} \} \subseteq N_{\mathfrak{M}}(v)$ such that $t' \leq c$ and  $(\mathfrak{M}, z_i) \models \varphi_{[\mathfrak{M},z_1]}^{\ell,c}$ for each $z_i \in Z$,
then there is 
a set $Z'=\{z_1', \dots, z_{t'}' \} \subseteq N_{\mathfrak{M'}}(v')$ of the same cardinality $t'$ such that $(\mathfrak{M}', z_i') \models \varphi_{[\mathfrak{M},z_1]}^{\ell,c}$ for each $z_i' \in Z'$.
Consider arbitrary $Z$ satisfying the required properties.
Since $(\mathfrak{M}, z_i) \models \varphi_{[\mathfrak{M},z_1]}^{\ell,c}$for each $z_i$, by the inductive assumption we obtain that $(\mathfrak{M}, z_i) \sim_{\ell,c} (\mathfrak{M}, z_j)$, for all $i,j \in \{1, \dots, t' \}$.
Hence, in the definition of $\varphi_{[\mathfrak{M},v]}^{\ell+1,c}$, we have $|U^\ell_{v,z_1}| \geq t'$, and so $\varphi_{[\mathfrak{M},v]}^{\ell+1,c}$ contains a conjunct $\Diamond_{t'} \varphi_{[\mathfrak{M},z_1]}^{\ell,c}$.
Since $(\mathfrak{M'},v') \models \varphi_{[\mathfrak{M},v]}^{\ell+1,c}$, we have therefore $(\mathfrak{M'},v') \models \Diamond_{t'} \varphi_{[\mathfrak{M},z_1]}^{\ell,c}$.
This, crucially, implies that required $Z'$ exists.
Recall that it implies existence of $U'$.
Now, the strategy for Duplicator is to choose this $U'$.
Then, Spoiler selects some $u_i' \in U'$.
The winning strategy for Duplicator is to select $u_i \in U$.
Indeed, we have $(\mathfrak{M'}, u_i') \models \varphi_{[\mathfrak{M},u_i]}^{\ell,c}$ so, by the inductive hypothesis, $(\mathfrak{M'}, u_i) \sim_{\ell,c} (\mathfrak{M}, u_i')$.
Hence, we get $(\mathfrak{M}, v) \sim_{\ell+1,c} (\mathfrak{M'}, v')$.

Assume that Spoiler selects a set $U' = \{u'_1,\dots, u'_t \}$ of size $t \leq c$ of elements in $N_{\mathfrak{M'}}(v')$. 
We will show that there exists a set $U = \{u_1,\dots, u_t \}$ of elements in $N_{\mathfrak{M}}(v)$
such that
$(\mathfrak{M'}, u_i') \models \varphi_{[\mathfrak{M},u_i]}^{\ell,c}$, for each $i \in \{ 1, \dots, t \}$.
To this end, it suffices to show that:
\begin{enumerate}[label=(\alph*)]
\item for each $z' \in N_{\mathfrak{M'}}(v')$ there exists $z \in N_{\mathfrak{M}}(v)$ such that  $(\mathfrak{M'}, z') \models \varphi_{[\mathfrak{M},z]}^{\ell,c}$, and 

\item if there is a set $Z'=\{z'_1, \dots, z'_{t'} \} \subseteq N_{\mathfrak{M'}}(v')$ and $z_1 \in N_{\mathfrak{M}}(v)$ such that $t' \leq c$ and  $(\mathfrak{M}, z'_i) \models \varphi_{[\mathfrak{M},z_1]}^{\ell,c}$ for each $z'_i \in Z$,
then there is 
a set $Z=\{z_1, \dots, z_{t'} \} \subseteq N_{\mathfrak{M}}(v)$ of the same cardinality $t'$ such that $\varphi_{[\mathfrak{M},z_i]}^{\ell,c} = \varphi_{[\mathfrak{M},z_1]}^{\ell,c}$ for each $z_i \in Z$.
\end{enumerate}
To show Statement (a), we observe that $(\mathfrak{M'},v') \models \varphi_{[\mathfrak{M},v]}^{\ell+1,c}$
implies that
$(\mathfrak{M'},v')  \models \neg \Diamond_1 \bigwedge_{\{v,w\} \in E} \neg \varphi_{[\mathfrak{M}, w]}^{\ell,c}$.
Hence, $(\mathfrak{M'},v')  \models \Box \bigvee_{\{v,w\} \in E}  \varphi_{[\mathfrak{M}, w]}^{\ell,c}$.
Therefore, for each $z' \in N_{\mathfrak{M'}}(v')$ there exists $w \in N_{\mathfrak{M}}(v)$ such that $(\mathfrak{M'}, z') \models \varphi_{[\mathfrak{M},w]}^{\ell,c}$, as required.
For Statement (b), consider arbitrary $Z'$ and $z_1$ satisfying the required properties.
By these properties, it follows directly that
$(\mathfrak{M'},v') \models \Diamond_{t'} \varphi_{[\mathfrak{M},z_1]}^{\ell+1,c}$.
We can show that it implies  $t' \leq |U^\ell_{v,z_1}|$.
Indeed,  if $|U^\ell_{v,z_1}| < t'$ then, by the  definition of  $\varphi_{[\mathfrak{M},v]}^{\ell+1,c}$, the formula   $\neg \Diamond_{t'} \varphi_{[\mathfrak{M},z_1]}^{\ell+1,c}$ would be a conjunct in $\varphi_{[\mathfrak{M},v]}^{\ell+1,c}$, and so, $(\mathfrak{M'},v') \models \varphi_{[\mathfrak{M},v]}^{\ell+1,c}$
would imply $(\mathfrak{M'},v') \models \neg \Diamond_{t'} \varphi_{[\mathfrak{M},z_1]}^{\ell+1,c}$, raising a contradiction.
Since $t' \leq |U^\ell_{v,z_1}|$, we obtain that 
$\Diamond_{t'} \varphi_{[\mathfrak{M},z_1]}^{\ell,c}$ is one of the conjuncts of 
$\varphi_{[\mathfrak{M},v]}^{\ell+1,c}
$, and so, $(\mathfrak{M},v) \models \Diamond_{t'} \varphi_{[\mathfrak{M},z_1]}^{\ell,c}$.
This implies that required $Z$ exists.

Recall that by showing Statements (a) and (b), we have shown existence of  required $U = \{u_1,\dots, u_t \}$.
Now, the strategy for Duplicator is to choose this set  $U$.
Then, Spoiler selects some $u_i \in U$.
The winning strategy for Duplicator is to select $u_i' \in U'$.
Indeed, we have $(\mathfrak{M'}, u_i') \models \varphi_{[\mathfrak{M},u_i]}^{\ell,c}$ so, by the inductive hypothesis, $(\mathfrak{M'}, u_i') \sim_{\ell,c} (\mathfrak{M}, u_i)$.
Hence, we get $(\mathfrak{M}, v) \sim_{\ell+1,c} (\mathfrak{M'}, v')$.
\end{proof}

\begin{example*}[Characteristic Formulas for $\GML$]
To emphasise why the characteristic formulas for $\GML$, as defined by  \citet{DBLP:journals/corr/abs-1910-00039} do not contain all the required information, we provide a concrete example.
Consider the a pointed model ($\mathfrak{M}, v$) over $\prop = \{p,q\}$ as depicted in \Cref{ex:Otto}.
If we apply the definition of \citet{DBLP:journals/corr/abs-1910-00039}, for $\ell = 1$ and $c = 3$, we would obtain the following characteristic formula for the node $v$:
\begin{align*}
\psi_{[\mathfrak{M},v]}^{1,3}  = & \;
(p \land \lnot q) \; \land  \; \\
&\Big( \Diamond_1 (p \land \lnot q)  \land  \Diamond_2 (p \land \lnot q) \land  \Diamond_1 (q \land \lnot p) \Big) \;\land \; \\
&\Big( \lnot \Diamond_3 (p \land \lnot q)   \land \lnot \Diamond_2 (q \land \lnot p) \land \lnot \Diamond_3 (q \land \lnot p) \Big) .
\end{align*}
The above formula, however, does not express, for example, that $v$ has no neighbour satisfying $\neg p \land \neg q$.
Hence  \Cref{th:game-GML} for such defined characteristic formulas does not hold. 
Indeed, consider the following pointed model $(\mathfrak{M'}, v')$ from \Cref{ex:Otto}.
Even though $(\mathfrak{M}',v') \models \psi_{[\mathfrak{M},v]}^{\ell,c}$, we still have   $(\mathfrak{M}, v) \not\equiv_{\GML_{\ell,c}} (\mathfrak{M'}, v')$ (and $(\mathfrak{M}, v) \not\sim_{\ell,c} (\mathfrak{M'}, v')$).
For example, we have $(\mathfrak{M}, v) \models \neg \Diamond_1 (\neg p \land \neg q)$, but $(\mathfrak{M}', v') \not\models \neg \Diamond_1 (\neg p \land \neg q)$.

\begin{figure*}[ht]
\centering
\begin{minipage}{.5\textwidth}
  \centering
\begin{center}
\tikzstyle{place}=[circle,draw=blue!50,fill=blue!20,thick]
\tikzstyle{transition}=[circle,draw=black!50,fill=black!20,thick]
\begin{tikzpicture}
\node at ( 0,3) (v4) [place,
                      label=right:$u_1$] {p};] {p};
\node at ( 1,1.5) (v3) [place,
                        label=right:$u_2$] {p};] {p};] {p};
\node at ( 0,0) (v2) [transition, 
                      label=right:$u_3$] {q};] {p};] {q};
\node at (-2,1.5) (v1) [place,
                   label=left:$\mathfrak{M}: ~~~v$] {p};
\draw [line width=0.3mm] (v1.east) -- (v3.west);
\draw [line width=0.3mm] (v1.east) -- (v2.west);
\draw [line width=0.3mm] (v1.east) -- (v4.west);
\end{tikzpicture}
\end{center}
\end{minipage}%
\begin{minipage}{.5\textwidth}
  \centering
\begin{center}
\tikzstyle{place}=[circle,draw=blue!50,fill=blue!20,thick]
\tikzstyle{transition}=[circle,draw=black!50,fill=black!20,thick]
\begin{tikzpicture}
\node at ( 0,3) (v4) [place,
                      label=right:$u'_1$] {p};] {p};
\node at ( -2,3) (v5) [place,
                      label=right:$u'_4$] {\phantom{a}};] {};
\node at ( 1,1.5) (v3) [place,
                        label=right:$u'_2$] {p};] {p};] {p};
\node at ( 0,0) (v2) [transition, 
                      label=right:$u'_3$] {q};] {p};] {q};
\node at (-2,1.5) (v1) [place,
                   label=left:$\mathfrak{M'}: ~~~v'$] {p};
\draw [line width=0.3mm] (v1.east) -- (v3.west);
\draw [line width=0.3mm] (v1.east) -- (v2.west);
\draw [line width=0.3mm] (v1.east) -- (v4.west);
\draw [line width=0.3mm] (v1.north) -- (v5.south);
\end{tikzpicture}
\end{center}
\end{minipage}
\caption{Models $\mathfrak{M}$ and $\mathfrak{M}'$}\label{ex:Otto}
\end{figure*}

We observe that our extended definition of characteristic formulas does not have this problem.
In particular, according to our definition, the characteristic formula for $v$ in $\mathfrak{M}$ is as follows:
\begin{align*}
\varphi_{[\mathfrak{M},v]}^{1,3}  = & \;
(p \land \lnot q) \; \land  \; \\
&\Big( \Diamond_1 (p \land \lnot q)  \land  \Diamond_2 (p \land \lnot q) \land  \Diamond_1 (q \land \lnot p) \Big) \;\land \; \\
&\Big( \lnot \Diamond_3 (p \land \lnot q)   \land \lnot \Diamond_2 (q \land \lnot p) \land \lnot \Diamond_3 (q \land \lnot p) \Big)  \;\land \; \\
&  \lnot \Diamond_1 \Big( \lnot (p \land \lnot q) \land \lnot (q \land \lnot p) \Big).
\end{align*}
It can be checked that  $(\mathfrak{M}',v') \not\models \varphi_{[\mathfrak{M},v]}^{\ell,c}$, as required.



\end{example*}

\biac*

\begin{proof}
To show Statement 1, consider  $\N \in \GNN{\bm}{\ac}$.
We show by induction on $\ell \leq L$ and fixed $k$ that $(G_1,v_1) \sim_{\ell,k} (G_2, v_2)$ implies  $\emb_1(v_1)^{(\ell)} = \emb_2(v_2)^{(\ell)}$,
where $G_1$, $G_2$ are graphs, $v_1$, $v_2$ any of their nodes, and $k \geq 1$. 
If $\ell = 0$, 
then $\ell \leq L$ that $(G_1,v_1) \sim_{\ell,k} (G_2, v_2)$ implies that $v_1$ and $v_2$ satisfy the same propositions, so $\emb_1(v_1)^{(\ell)} = \emb_2(v_2)^{(\ell)}$.

For the induction step, assume that $(G_1,v_1) \sim_{\ell-1,k} (G_2, v_2)$ implies $\emb_1(v_1)^{(\ell-1)} = \emb_2(v_2)^{(\ell-1)}$. 
We will prove this statement  for $\ell$, by showing  contrapositive of the implication.
Hence, we assume  that $\emb_1(v_1)^{(\ell)} \neq \emb_2(v_2)^{(\ell)}$. Recall that by Equation \ref{eq:ac} we have 
\begin{align*}
\emb_i(v_i)^{(\ell)} := \comb \Big( \emb_i(v_i)^{(\ell-1)}, \agg(\lBrace   \emb_i(w)^{(\ell-1)}  \rBrace _{w \in N_{\G}(v_i)})\Big), 
\end{align*}
\noindent so one of the following two cases needs to hold:
\begin{center}
\begin{tabular}{l}
     1. $\emb_1(v_1)^{(\ell-1)} \neq \emb_2(v_2)^{(\ell-1)}$, or\\
     2. $\agg(\lBrace  \emb_1(w)^{(\ell-1)} \rBrace_{w \in N_{\G_1}(v_1)}) \neq \agg( \lBrace  \emb_2(w)^{(\ell-1)} \rBrace_{w \in N_{\G_2}(v_2)})$.
\end{tabular}
\end{center}

Notice that (1) and the inductive hypothesis immediately imply that $(G_1,v_1) \not\sim_{\ell-1,k} (G_2, v_2)$, so we get $(G_1,v_1) \not\sim_{\ell,k} (G_2, v_2)$. If (2), then since $\agg$ is $k$-bounded, there is $w_1 \in N_{G_1}(v_1)$ such that $\emb_1(w_1)^{(\ell-1)}$ occurs $k_1$ times in $\lBrace  \emb_1(w)^{(\ell-1)} \rBrace_{w \in N_{G_1}(v_1)}$
and 
$k_2 \neq k_1$  times
in $ \lBrace  \emb_2^{(\ell-1)}(w) \rBrace_{w \in N_{\G_2}(v_2)}$ such that either $k_1 < k$ or $k_2 < k$.
Assume w.l.o.g. that $k_2 < k_1$.
We will that Spoiler has a winning strategy as follows.
He should select a set $U_1$ of $\mathit{min}(k,k_1)$ elements $w \in N_{G_1}(v_1)$ satisfying $\emb_1^{(\ell-1)}(w) = \emb_1^{(\ell-1)}(w_1)$.
Then, Duplicator must respond with a subset $U_2$ of $w \in N_{G_2}(v_2)$ with $|U_2| = |U_1| = \mathit{min}(k,k_1)$.
Any such $U_2$  must contain a $w_2$ such that $\emb_2^{(\ell-1)}(w_2) \neq \emb_1^{(\ell-1)}(w_1)$. 
Then Spoiler has a winning strategy by choosing $w_2$ since whichever $w_1' \in U_1$ Duplicator chooses, we have  $\emb_2^{(\ell-1)}(w_2) \neq \emb_1^{(\ell-1)}(w_1')$.
By the inductive hypothesis,  $(G_1,w'_1) \not\sim_{\ell-1, k} (G_2, w_2)$, and so $(G_1,v_1) \not\sim_{\ell, k} (G_2, v_2)$.

If $\N \in \GNN{\s}{\ac}$, the aggregation function $\agg$ is set-based, so it suffices to consider the proof for $\N \in \GNN{\bm}{\ac}$ under the special case where $k=1$. 
\end{proof}

\newpage

\section*{Proof Details for Section \ref{sec:readouts}}

\linalgacr*

\begin{proof}
Let $\varphi \in \GMLC$ be a logic classifier of dimension $d$ with propositions $\prop(\varphi)$ and subformulas $\mathsf{sub}(\varphi) = (\varphi_1, \ldots, \varphi_L)$, where $k \leq \ell$ whenever $\varphi_k$ is a subformula of $\varphi_{\ell}$.
We construct a GNN classifier $\N_{\varphi}$ with layers $0, \ldots, L$.
The classification function $\cls$ is such that $\cls(\mathbf{x}) = \true$ iff the last element of $\mathbf{x}$ is $1$.
Layer $0$  uses a combination function to multiply input vectors by a  matrix $\mathbf{D} \in \R^{d \times L}$, namely  $\lambda(v)^{(1)} =  \lambda(v)\mathbf{D} $, where $\mathbf{D}_{k \ell} = 1$ if the $k$th position of the  input vectors  corresponds to a proposition $\varphi_\ell$;
other entries of $\mathbf{D}$ are $0$.
Other layers are of dimension $L$ and defined below. All these remaining layers are homogeneous, meaning that the aggregation, combination, and readout functions 
are the same (i.e., they share the same parameters) across all layers. The activation function is the componentwise truncated ReLU defined as $\sigma(x) = \mathit{min}(\mathit{max}(0,x),1)$. 
The combination function is $\comb(\mathbf{x}, \mathbf{y}, \mathbf{z}) = \sigma(\mathbf{x} \mathbf{C}  + \mathbf{y} \mathbf{A} + \mathbf{z} \mathbf{R} + \mathbf{b})$, whereas  the aggregation and readout functions, $\agg$ and $\readout$, are the max-$n$-sum, where $n$ is the counting rank of $\varphi$. 
The entries of $\mathbf{A}, \mathbf{C}, \mathbf{R} \in \mathbb{R}^{L \times L}$ and  $\mathbf{b} \in \mathbb{R}^L$ depend on the subformulas of $\varphi$ as follows: 

\begin{center}
\begin{tabular}{l}
1. if $\varphi_{\ell}$ is a proposition, $C_{\ell\ell} = 1$, \\
2. if $\varphi_{\ell} = \varphi_j \wedge \varphi_k$, then $C_{j \ell} = C_{k \ell} =1$ and $b_{\ell} = -1$; \\
3. if $\varphi_{\ell} =\neg \varphi_k$, then $C_{k\ell} = -1$ and $b_{\ell} = 1$;\\
4. if $\varphi_{\ell} = \Diamond_c \varphi_k$, then $A_{k\ell} = 1$ and $b_{\ell} = -c + 1$;\\
5. if $\varphi_\ell = \countE_c \varphi_{k}$, then $R_{k\ell} = 1$ and $b_\ell = -c+1$; and
\end{tabular}
\end{center}

\noindent all other entries are set to $0$.

Consider the GNN application to $\G = (V,E,\emb)$ of dimension $d$.
We  show that for each subformula $\varphi_{\ell}$, each  $i \in \{\ell, \ldots, L\}$, and  $v \in V$, if   $(\M_{\G},v) \models \varphi_{\ell}$ then $\emb(v)^{(i)}_{\ell} =1 $, and  otherwise $\emb(v)^{(i)}_{\ell} =0$. 
This implies that $\N_\varphi(\G,v) = \true$ iff $(\M_{\G},v) \models \varphi$.
The proof is by induction on the structure of $\varphi_\ell$.

\begin{itemize}
\item If $\varphi_{\ell}$ is a proposition then, by the design of layer $0$, we have  $\lambda(v)^{(0)}_\ell =1$ if $(\M_{G},v) \models \varphi_\ell$, and otherwise $\lambda(v)^{(0)}_\ell =0$.
Moreover, since $\varphi_\ell$ is a proposition,
all  layers have $C_{\ell \ell} = 1$, $b_{\ell} = 0$, $A_{k\ell} = 0$, and $R_{k\ell} = 0$ for each $k$, so
$\comb(\mathbf{x}, \mathbf{y}, \mathbf{z})_{\ell} = \mathbf{x}_{\ell}$.
Thus $\lambda(v)^{(0)}_\ell = \lambda(v)^{(i)}_\ell$ for any $i$.

\item If $\varphi_{\ell} = \neg \varphi_k$, then by construction, we have $C_{k\ell} = -1$, $b_\ell = 1$, $R_{m\ell} = 0$, and $A_{m\ell} = 0$ for each $m$. 
Then by Equation~\eqref{eq:integrated}, we have $\emb(v)^{(i)}_\ell = \sigma(-\emb(v)^{(i-1)}_k + 1)$.
By inductive hypothesis, 
$\lambda(v)^{(i-1)}_k = 1$ if  $(\M_{G},v) \models \varphi_k$,  and  otherwise $\lambda(v)^{(i-1)}_k = 0$. 
Hence, 
$\emb(v)^{(i)}_\ell=1$ if $(\M_{\G} , v) \models  \neg \varphi_k$, and otherwise $\emb(v)^{(i)}_\ell=0$ for all $i \geq \ell$.

\item If $\varphi_{\ell} = \varphi_j \wedge \varphi_k$, then by construction, we have $C_{j\ell} = C_{k\ell} = 1$, $b_\ell = -1$, $R_{m\ell} = 0$, and $A_{m\ell} = 0$ for each $m$.
Then by Equation~\eqref{eq:integrated}, we have $\emb(v)^{(i)}_\ell = \sigma(\emb(v)^{(i-1)}_j + \emb(v)^{(i-1)}_k - 1)$.
By inductive hypothesis, 
$\lambda(v)^{(i-1)}_k = 1$ if  $(\M_{G},v) \models \varphi_k$,  and  otherwise $\lambda(v)^{(i-1)}_k = 0$; and $\lambda(v)^{(i-1)}_j = 1$ if  $(\M_{G},v) \models \varphi_j$, and  otherwise $\lambda(v)^{(i-1)}_j = 0$.
Hence, 
$\emb(v)^{(i)}_\ell=1$ if $(\M_{\G} , v) \models  \varphi_j \wedge \varphi_k$, and otherwise $\emb(v)^{(i)}_\ell=0$ for all $i \geq \ell$.

\item If $\varphi_{\ell} = \Diamond_c \varphi_k$, then by construction, we have $A_{k \ell} = 1$,  $b_{\ell} = -c+1$, $R_{m\ell} = 0$, and $C_{m \ell} = 0$ for each $m$. 
Since \agg{} is the max-$n$-sum, by Equation~\eqref{eq:integrated}, we have $\emb(v)^{(i)}_\ell = \sigma(\mathit{min}(n, \sum_{w\in N_G(v)} \lambda(w)^{(i-1)}_k) -c + 1)$.
By inductive hypothesis, 
$\lambda(w)^{(i-1)}_k = 1$ if  $(\M_{G},w) \models \varphi_k$,  and  otherwise $\lambda(w)^{(i-1)}_k = 0$. 
Hence, 
$\emb(v)^{(i)}_\ell=1$ if $(\M_{\G} , v) \models \Diamond_c \varphi_k$, and otherwise $\emb(v)^{(i)}_\ell=0$.

\item If $\varphi_\ell = \countE_c \varphi_{k}$, then by construction, we have $R_{k \ell} = 1$,  $b_{\ell} = -c+1$,  $C_{m \ell} = 0$, and $A_{m \ell} = 0$ for each $m$. 
Since \agg{} is max-$n$-sum, by Equation~\eqref{eq:integrated}, we have $\emb(v)^{(i)}_\ell = \sigma(\mathit{min}(n, \sum_{w\in V_G} \emb(w)^{(i-1)}_k) -c + 1)$.
By the inductive hypothesis, 
$\emb(w)^{(i-1)}_k = 1$ if  $(\M_{G},w) \models \varphi_k$,  and  otherwise $\emb(w)^{(i-1)}_k = 0$. 
Hence, 
$\emb(v)^{(i)}_\ell=1$ if $(\M_{\G} , v) \models \countE_c \varphi_{k}$, and otherwise $\emb(v)^{(i)}_\ell=0$.
\end{itemize}

Finally, observe that $\MLA$ is a fragment of $\GMLC$ where the counting rank is restricted to $1$. 
Then, the max-$n$-sum function becomes the max-$1$-sum, which is simply the component-wise maximum function.
Hence, we have $\MLA \leq \GNN{\s}{\acr}$ when the counting rank is 1.
\end{proof}

\begin{example*}[$\GNN{\bm}{\acr}$ 
Construction for $\GMLC$]
We illustrate the construction of $\N \in \GNN{\bm}{\acr}$ from \Cref{prop:ml-to-gnn-acr} for the $\GML_{\ell,c}$ formula $\varphi = \Diamond_3 p \land \exists_3 \lnot \Diamond_3 q$ with quantifier rank $\ell = 2$ and counting rank $c = 3$.
We  assume that input graphs have labellings of the form $[x_p, x_q]$, corresponding to the node colours $p$ and $q$. The first step is to provide an ordering of subformulas of $\varphi$ as described in the proof of \Cref{prop:ml-to-gnn-acr}, for example: $\varphi_1 = p$, $\varphi_2 = q$, $\varphi_3 = \Diamond_3 p$, $\varphi_4 = \Diamond_3 q$, $\varphi_5 = \lnot \Diamond_3 q$, $\varphi_6 = \exists_3 \lnot \Diamond_3 q$, and $\varphi_7 = \varphi = \Diamond_3 p \land \exists_3 \lnot \Diamond_3 q$.


As before, we construct matrix $\mathbf{D}$ for Layer 0 to extend the input labels: 

$$\mathbf{D} = 
\begin{bmatrix}
1 & 0 & 0 & 0 & 0 & 0 & 0\\
0 & 1 & 0 & 0 & 0 & 0 & 0
\end{bmatrix}.
$$

We use $\mathbf{C}$ and $\mathbf{A}$ to propagate conjunctions, negations and local graded modalities; use $\mathbf{b}$ for offsetting multiplicities. Following the construction from Theorem \ref{prop:ml-to-gnn-acr}, we get:

$$\mathbf{C} = 
\begin{bmatrix}
1 & 0 & 0 & 0 & 0 & 0 & 0\\
0 & 1 & 0 & 0 & 0 & 0 & 0\\
0 & 0 & 0 & 0 & 0 & 0 & 1\\
0 & 0 & 0 & 0 & -1 & 0 & 0\\
0 & 0 & 0 & 0 & 0 & 0 & 0\\
0 & 0 & 0 & 0 & 0 & 0 & 1 \\
0 & 0 & 0 & 0 & 0 & 0 & 0
\end{bmatrix}, ~~
\mathbf{A} = 
\begin{bmatrix}
0 & 0 & 1 & 0 & 0 & 0 & 0\\
0 & 0 & 0 & 1 & 0 & 0 & 0\\
0 & 0 & 0 & 0 & 0 & 0 & 0\\
0 & 0 & 0 & 0 & 0 & 0 & 0\\
0 & 0 & 0 & 0 & 0 & 0 & 0\\
0 & 0 & 0 & 0 & 0 & 0 & 0\\
0 & 0 & 0 & 0 & 0 & 0 & 0
\end{bmatrix}, \text{ and } ~
\mathbf{b} = 
\begin{bmatrix}
0 & 0 & -2 & -2 & 1 & -2 & -1
\end{bmatrix}.
$$

To propagate the global counting modality $\exists_k$ in $\GMLC$, we use matrix $\mathbf{R}$ to process the aggregation over the set of all vertices. By construction from Theorem \ref{prop:ml-to-gnn-acr}, we get:

$$\mathbf{R} = 
\begin{bmatrix}
0 & 0 & 0 & 0 & 0 & 0 & 0\\
0 & 0 & 0 & 0 & 0 & 0 & 0\\
0 & 0 & 0 & 0 & 0 & 0 & 0\\
0 & 0 & 0 & 0 & 0 & 0 & 0\\
0 & 0 & 0 & 0 & 0 & 1 & 0\\
0 & 0 & 0 & 0 & 0 & 0 & 0\\
0 & 0 & 0 & 0 & 0 & 0 & 0
\end{bmatrix}.
$$
\end{example*}

\gameGMLC*

\begin{proof}
To prove the theorem we will show three implications: 

\begin{center}
\begin{tabular}{l}
1. $(\mathfrak{M}, v) \sim_{\ell,c}^{\exists} (\mathfrak{M'}, v')$ implies  $(\mathfrak{M}, v) \equiv_{\GMLC_{\ell,c}} (\mathfrak{M'}, v')$, \\
2. $(\mathfrak{M}, v) \equiv_{\GMLC_{\ell,c}} (\mathfrak{M'}, v')$ implies $(\mathfrak{M}',v') \models \varphi_{\exists[\mathfrak{M},v]}^{\ell,c}$, and \\
3. $(\mathfrak{M}',v') \models \varphi_{\exists[\mathfrak{M},v]}^{\ell,c}$ implies $(\mathfrak{M}, v) \sim_{\ell,c}^{\exists} (\mathfrak{M'}, v')$.
\end{tabular}
\end{center}

For the first implication we prove  by induction on $\ell$ that for any fixed $c$,  $(\mathfrak{M},v) \not\equiv_{\GMLC_{\ell,c}} (\mathfrak{M'}, v')$
implies $(\mathfrak{M}, v) \not\sim_{\ell,c}^{\exists} (\mathfrak{M'}, v')$. 
For the base case $\ell = 0$, $(\mathfrak{M}, v) \sim_{0,c}^{\exists} (\mathfrak{M'}, v')$ in the $0$-th round implies that $v$ and $v'$ satisfy the same propositions, so we have $(\mathfrak{M},v) \equiv_{\GMLC_{0,c}} (\mathfrak{M'}, v')$.
For the inductive step, assume $(\mathfrak{M},v) \not\equiv_{\GMLC_{\ell,c}} (\mathfrak{M'}, v')$ implies $(\mathfrak{M}, v) \not\sim_{\ell,c}^{\exists} (\mathfrak{M'}, v')$. 
We show that the claim holds for $\ell+1$. Assume $(\mathfrak{M},v) \not\equiv_{\GMLC_{\ell+1,c}} (\mathfrak{M'}, v')$, then there exists $\varphi \in \GMLC_{\ell+1,c}$ such that $(\mathfrak{M},v) \models \varphi$ but $(\mathfrak{M'},v') \not\models \varphi$. 
It suffices to show the case for $\varphi = \countE_k \phi$, where $k \leq c$ and $\phi \in \GMLC_{\ell, c}$. 
The proofs for the remaining cases are identical to those in \cite{DBLP:journals/corr/abs-1910-00039}. 
Spoiler has a winning strategy from $(\mathfrak{M},v,\mathfrak{M'},v')$ by picking a set $U_1$ of $k$ worlds in $\mathfrak{M}$ such that for each $u \in U_1$, $(\mathfrak{M}, u) \models \phi$. 
Then, any subset $U_2$ of size $k$ picked by Duplicator must contain some $u'$ such that $(\mathfrak{M'},u') \not\models \phi$. 
Spoiler can subsequently pick $u'$ as the next pebble position.
Then by the inductive hypothesis, we have $(\mathfrak{M}, u) \not\sim_{\ell,c}^{\exists} (\mathfrak{M'}, u')$, so we get $(\mathfrak{M}, v) \not\sim_{\ell+1,c}^{\exists} (\mathfrak{M'}, v')$.

For the second implication we observe that
the characteristic formula $\varphi_{\exists[\mathfrak{M},v]}^{\ell,c}$ is constructed so that 
$(\mathfrak{M},v) \models \varphi_{\exists[\mathfrak{M},v]}^{\ell,c}$.
Moreover, $\varphi_{\exists[\mathfrak{M},v]}^{\ell,c}$ is a $\GMLC_{\ell,c}$ formula, so
$(\mathfrak{M}, v) \equiv_{\GMLC_{\ell,c}} (\mathfrak{M'}, v')$ immediately implies $(\mathfrak{M}',v') \models \varphi_{\exists[\mathfrak{M},v]}^{\ell,c}$.

Now, we show the third implication  by induction on $\ell$, for a fixed $c$. 
For the base case when $\ell = 0$, $(\mathfrak{M'},v') \models \varphi_{\exists[\mathfrak{M},v]}^{0,c}$ implies that $v$ and $v'$ satisfy the same propositions, so we have $(\mathfrak{M}, v) \sim_{0,c}^{\exists} (\mathfrak{M'}, v')$. 
For the inductive step, assume $(\mathfrak{M'},v') \models \varphi_{\exists[\mathfrak{M},v]}^{\ell,c}$ implies $(\mathfrak{M}, v) \sim_{\ell,c}^{\exists} (\mathfrak{M'}, v')$.
We will show that  if $(\mathfrak{M'},v') \models \varphi_{\exists[\mathfrak{M},v]}^{\ell+1,c}$, then $(\mathfrak{M}, v) \sim_{\ell+1,c}^{\exists} (\mathfrak{M'}, v')$.
To this end, we will show that  Duplicator has a winning strategy.
First, we need to show that $(\mathfrak{M},v,\mathfrak{M}',v')$ is not a winning configuration for Spoiler, that is, $(\mathfrak{M}, v)$ and $(\mathfrak{M'}, v')$ satisfy the same propositions.
Indeed, by the form of $\varphi_{\exists[\mathfrak{M},v]}^{\ell,c}$, the fact that $(\mathfrak{M'},v') \models \varphi_{\exists[\mathfrak{M},v]}^{\ell+1,c}$ implies  $(\mathfrak{M'},v') \models \varphi_{\exists[\mathfrak{M},v]}^{0,c}$.
The latter, implies that $(\mathfrak{M}, v)$ and $(\mathfrak{M'}, v')$ satisfy the same propositions.
To show how the winning strategy for Duplicatior, we need to consider all possibilities for Spoiler moves. If Spoiler plays a local round, then the proof proceeds as in that of Theorem \ref{th:game-GML}, and we get $(\mathfrak{M}, v) \sim_{\ell+1,c}^{\exists} (\mathfrak{M'}, v')$.
Next, we will consider separately  the cases when Spoiler plays a set in $\mathfrak{M}$ and in $\mathfrak{M}'$ in a global round.

Assume that Spoiler selects a set $U = \{u_1,\dots, u_t \}$ of size $t \leq c$ of elements in $\mathfrak{M}$. 
We will show that there exists a set $U' = \{u_1',\dots, u_t' \}$ of elements in $\mathfrak{M'}$
such that
$(\mathfrak{M'}, u_i') \models \varphi_{\exists[\mathfrak{M},u_i]}^{\ell,c}$, for each $i \in \{ 1, \dots, t \}$.
To this end, it suffices to show that if there is a set $Z=\{z_1, \dots, z_{t'} \} \subseteq V$ such that $t' \leq c$ and  $(\mathfrak{M}, z_i) \models \varphi_{\exists[\mathfrak{M},z_1]}^{\ell,c}$ for each $z_i \in Z$,
then there is 
a set $Z'=\{z_1', \dots, z_{t'}' \} \subseteq V'$ of the same cardinality $t'$ such that $(\mathfrak{M}', z_i') \models \varphi_{\exists[\mathfrak{M},z_1]}^{\ell,c}$ for each $z_i' \in Z'$.
Consider arbitrary $Z$ satisfying the required properties.
Since $(\mathfrak{M}, z_i) \models \varphi_{\exists[\mathfrak{M},z_1]}^{\ell,c}$for each $z_i$, by the inductive assumption we obtain that $(\mathfrak{M}, z_i) \sim_{\ell,c}^{\exists} (\mathfrak{M}, z_j)$, for all $i,j \in \{1, \dots, t' \}$.
Hence, in the definition of $\varphi_{\exists[\mathfrak{M},v]}^{\ell+1,c}$, we have $|J^\ell_{z_1}| \geq t'$, and so $\varphi_{\exists[\mathfrak{M},v]}^{\ell+1,c}$ contains a conjunct $\countE_{t'} \varphi_{\exists[\mathfrak{M},z_1]}^{\ell,c}$.
Since $(\mathfrak{M'},v') \models \varphi_{\exists[\mathfrak{M},v]}^{\ell+1,c}$, we have therefore $(\mathfrak{M'},v') \models \countE_{t'} \varphi_{\exists[\mathfrak{M},z_1]}^{\ell,c}$.
This, crucially, implies that required $Z'$ exists.
Recall that it implies existence of $U'$.
Now, the strategy for Duplicator is to choose this $U'$.
Then, Spoiler selects some $u_i' \in U'$.
The winning strategy for Duplicator is to select $u_i \in U$.
Indeed, we have $(\mathfrak{M'}, u_i') \models \varphi_{\exists[\mathfrak{M},u_i]}^{\ell,c}$ so, by the inductive hypothesis, $(\mathfrak{M'}, u_i) \sim_{\ell,c}^{\exists} (\mathfrak{M}, u_i')$.
Hence, we get $(\mathfrak{M}, v) \sim_{\ell+1,c}^{\exists} (\mathfrak{M'}, v')$.

Assume that Spoiler selects a set $U' = \{u'_1,\dots, u'_t \}$ of size $t \leq c$ of elements in $V$. 
We will show that there exists a set $U = \{u_1,\dots, u_t \}$ of elements in $V'$
such that
$(\mathfrak{M'}, u_i') \models \varphi_{\exists[\mathfrak{M},u_i]}^{\ell,c}$, for each $i \in \{ 1, \dots, t \}$.
To this end, it suffices to show that:
\begin{enumerate}[label=(\alph*)]
\item for each $z' \in V'$ there exists $z \in V$ such that  $(\mathfrak{M'}, z') \models \varphi_{\countE[\mathfrak{M},z]}^{\ell,c}$, and 

\item if there is a set $Z'=\{z'_1, \dots, z'_{t'} \} \subseteq V'$ and $z_1 \in V$ such that $t' \leq c$ and  $(\mathfrak{M}, z'_i) \models \varphi_{\countE[\mathfrak{M},z_1]}^{\ell,c}$ for each $z'_i \in Z$,
then there is 
a set $Z=\{z_1, \dots, z_{t'} \} \subseteq V$ of the same cardinality $t'$ such that $\varphi_{\countE[\mathfrak{M},z_i]}^{\ell,c} = \varphi_{\countE[\mathfrak{M},z_1]}^{\ell,c}$ for each $z_i \in Z$.
\end{enumerate}
To show Statement (a), we observe that $(\mathfrak{M'},v') \models \varphi_{\countE[\mathfrak{M},v]}^{\ell+1,c}$
implies that
$(\mathfrak{M'},v')  \models \neg \countE_1 \bigwedge_{w \in V} \neg \varphi_{[\mathfrak{M}, w]}^{\ell,c}$.
Therefore, for each $z' \in V'$ there exists $w \in V$ such that $(\mathfrak{M'}, z') \models \varphi_{\countE[\mathfrak{M},w]}^{\ell,c}$, as required.
For Statement (b), consider arbitrary $Z'$ and $z_1$ satisfying the required properties.
By these properties, it follows directly that
$(\mathfrak{M'},v') \models \countE_{t'} \varphi_{\countE[\mathfrak{M},z_1]}^{\ell+1,c}$.
We can show that it implies  $t' \leq |J^\ell_{z_1}|$.
Indeed,  if $|J^\ell_{z_1}| < t'$ then, by the  definition of  $\varphi_{\countE[\mathfrak{M},v]}^{\ell+1,c}$, the formula   $\neg \countE_{t'} \varphi_{\countE[\mathfrak{M},z_1]}^{\ell+1,c}$ would be a conjunct in $\varphi_{\countE[\mathfrak{M},v]}^{\ell+1,c}$, and so, $(\mathfrak{M'},v') \models \varphi_{\countE[\mathfrak{M},v]}^{\ell+1,c}$
would imply $(\mathfrak{M'},v') \models \neg \countE_{t'} \varphi_{\countE[\mathfrak{M},z_1]}^{\ell+1,c}$, raising a contradiction.
Since $t' \leq |J^\ell_{z_1}|$, we obtain that 
$\countE_{t'} \varphi_{\countE[\mathfrak{M},z_1]}^{\ell,c}$ is one of the conjuncts of 
$\varphi_{\countE[\mathfrak{M},v]}^{\ell+1,c}
$, and so, $(\mathfrak{M},v) \models \countE_{t'} \varphi_{\countE[\mathfrak{M},z_1]}^{\ell,c}$.
This implies that required $Z$ exists.

Recall that by showing Statements (a) and (b), we have shown existence of  required $U = \{u_1,\dots, u_t \}$.
Now, the strategy for Duplicator is to choose this set  $U$.
Then, Spoiler selects some $u_i \in U$.
The winning strategy for Duplicator is to select $u_i' \in U'$.
Indeed, we have $(\mathfrak{M'}, u_i') \models \varphi_{\countE[\mathfrak{M},u_i]}^{\ell,c}$ so, by the inductive hypothesis, $(\mathfrak{M'}, u_i') \sim_{\ell,c}^{\countE} (\mathfrak{M}, u_i)$.
Hence, we get $(\mathfrak{M}, v) \sim_{\ell+1,c}^{\countE} (\mathfrak{M'}, v')$.
\end{proof}

\begin{example*}[Characteristic Formulas for $\GMLC$]
Consider the following pointed model ($\mathfrak{M}, v$) over $\prop = \{p,q\}$ again:

\vspace{5mm}
\begin{center}
\tikzstyle{place}=[circle,draw=blue!50,fill=blue!20,thick]
\tikzstyle{transition}=[circle,draw=black!50,fill=black!20,thick]
\begin{tikzpicture}
\node at ( 0,3) (v4) [place,
                      label=right:$u_1$] {p};] {p};
\node at ( 1,1.5) (v3) [place,
                        label=right:$u_2$] {p};] {p};] {p};
\node at ( 0,0) (v2) [transition, 
                      label=right:$u_3$] {q};] {p};] {q};
\node at (-2,1.5) (v1) [place,
                   label=left:$\mathfrak{M}: ~~~v$] {p};
\draw [line width=0.3mm] (v1.east) -- (v3.west);
\draw [line width=0.3mm] (v1.east) -- (v2.west);
\draw [line width=0.3mm] (v1.east) -- (v4.west);
\end{tikzpicture}
\end{center}
\vspace{5mm}


Then, a characteristic formula, as in  \Cref{th:game-GMLC}, is as follows:
\begin{align*}
\varphi_{\Ecount[\mathfrak{M},v]}^{1,3} = & \;
(p \land \lnot q) \; \land \\
&\Big( \Diamond_1 (p \land \lnot q)  \land  \Diamond_2 (p \land \lnot q) \land  \Diamond_1 (q \land \lnot p) \Big) \;\land \; \\
&\Big( \lnot \Diamond_3 (p \land \lnot q)   \land \lnot \Diamond_2 (q \land \lnot p) \land \lnot \Diamond_3 (q \land \lnot p) \Big)  \;\land \; \\
&\Big( \exists_1 (p \land \lnot q) \land \exists_2 (p \land \lnot q) \land \exists_3 (p \land \lnot q) \land  \exists_1 (q \land \lnot p) \Big) \; \land \; \\
& \Big(\lnot \exists_2 (q \land \lnot p) \land \lnot \exists_3 (q \land \lnot p) \Big) \; \land \;\\
&  \lnot \Diamond_1 \Big( \lnot (p \land \lnot q) \land \lnot (q \land \lnot p) \Big) \;\land\; \lnot \exists_1 \Big( \lnot (p \land \lnot q) \land \lnot (q \land \lnot p) \Big).
\end{align*}
\end{example*}

\acrbi*

\begin{proof}
We first consider $\N \in \GNN{\bm}{\acr}$.
We show by induction on $\ell \leq L$ that $(G_1,v_1) \sim_{\ell,k}^{\exists} (G_2, v_2)$ implies  $\emb_1(v_1)^{(\ell)} = \emb_2(v_2)^{(\ell)}$,
where $G_1$, $G_2$ are graphs, $v_1$, $v_2$ any of their nodes, and $k \geq 1$. 
If $\ell = 0$, 
then $v_1$ and $v_2$ satisfy the same propositions, so $\emb_1(v_1)^{(\ell)} = \emb_2(v_2)^{(\ell)}$.
For the induction step, assume $(G_1,v_1) \sim_{\ell-1,k}^{\exists} (G_2, v_2)$ implies $\emb_1(v_1)^{(\ell-1)} = \emb_2(v_2)^{(\ell-1)}$. We prove the contrapositive for $\ell$.

If $\N \in \GNN{\bm}{\acr}$, then $\agg$ and $\readout$ are $k$-bounded. Assume that $\emb_1(v_1)^{(\ell)} \neq \emb_2(v_2)^{(\ell)}$. Then by Equation \ref{eq:integrated}, 

\begin{align*}
\emb_i(v_i)^{(\ell)} := \comb \Big( \emb_i(v_i)^{(\ell-1)}, &\agg(\lBrace   \emb_i(w)^{(\ell-1)}  \rBrace _{w \in N_{\G}(v_i)}),\readout( \lBrace  \emb_i(w)^{(\ell-1)} \rBrace_{w \in V})\Big), 
\end{align*}

\noindent so we have one of the following:  

\begin{center}
\begin{tabular}{l}
1. $\emb_1(v_1)^{(\ell-1)} \neq \emb_2(v_2)^{(\ell-1)}$,\\
2. $\agg(\lBrace   \emb_1(w)^{(\ell-1)}  \rBrace _{w \in N_{\G_1}(v_1)}) \neq \agg(\lBrace   \emb_2(w)^{(\ell-1)}  \rBrace _{w \in N_{\G_2}(v_2)})$, or\\
3. $\readout( \lBrace  \emb_1(w)^{(\ell-1)} \rBrace_{w \in V_1}) \neq \readout( \lBrace  \emb_2(w)^{(\ell-1)} \rBrace_{w \in V_2})$.
\end{tabular}
\end{center}

Notice that (1) and the inductive hypothesis immediately imply that $(G_1,v_1) \not\sim_{\ell-1,k}^{\exists} (G_2, v_2)$, so we get $(G_1,v_1) \not\sim_{\ell,k}^{\exists} (G_2, v_2)$. 
(2) corresponds to the $k$-bounded $\ac$ case in the proof of Theorem \ref{thm:bisimulation-invariance} and implies $(G_1,v_1) \not\sim_{\ell,k}^{\exists} (G_2, v_2)$. 
Lastly, since the readout function $\readout$ is $k$-bounded, (3) implies that there exists $w_1 \in V_1$ such that $\emb_1(w_1)^{(\ell-1)}$ appears $k_1$ times in $\lBrace  \emb_1(w)^{(\ell-1)} \rBrace_{w \in V_1}$
and 
$k_2 \neq k_1$  times
in $ \lBrace  \emb_2^{(\ell-1)}(w) \rBrace_{w \in V_2}$ such that either $k_1 < k$ or $k_2 < k$.
Assume w.l.o.g. $k_2 < k_1$.
Then,  Spoiler can pick a set $U_1$ of $\mathit{min}(k,k_1)$ worlds in $V_1$ with label $\emb_1(w_1)^{(\ell)}$. 
Then any set $U_2$ of $\mathit{min}(k,k_1)$ worlds  Duplicator picks in $V_2$ contains a world $u$ such that  $\emb_2(u)^{(\ell-1)} \neq \emb_1(w_1)^{(\ell-1)}$. 
Then by the inductive hypothesis, we have $(G_1, u') \not\sim_{\ell-1,k}^{\exists} (G_2, u)$ for every $u' \in U_1$. Hence,  Spoiler has a winning strategy by picking $U_2$ then selecting the required $u$ from any $U_2$ Duplicator selects. 
Hence, we have $(G_1, v_1) \not\sim_{\ell,k}^{\exists} (G_2, v_2)$.

Finally, the proof for $\sim_{\ell,k}^{\exists}$-invariance of $\GNN{\s}{\acr}$ with $L$ layers can be treated as a special case of the proof for $\GNN{\bm}{\acr}$ under the special case where $k = 1$.
\end{proof}

\newpage
\section*{Proof Details for Section \ref{sec:two-var}}

To show that $\GNN{\bm}{\acn}$  and $\GNN{\s}{\acn}$  capture $\Ctwo$ and  $\FOtwo$  respectively, we will exploit the technique  \citet{DBLP:conf/iclr/BarceloKM0RS20}, based on the fact that \FOtwo{} has the same expressive power as the modal logic $\MLpar$ with complex modalities \cite[Theorem 1]{DBLP:conf/csl/LutzSW01} and 
\Ctwo{} over coloured graphs as  \EMLC{} \cite[Theorem D.3]{DBLP:conf/iclr/BarceloKM0RS20}. 

The grammar for \EMLC{} formulas (in normal form)
is given below \cite[Lemma D.4]{DBLP:conf/iclr/BarceloKM0RS20}:
$$
\varphi :=  p \mid \neg \varphi \mid \varphi \land \varphi \mid \langle S \rangle ^{\geq c} \varphi,   
$$
where $p$ ranges over propositions, $c \in \mathbb{N}$, and $S$ is a modal parameter from the following grammar:
$$
S:= id  \mid  \neg id \mid e \mid \neg e 
\mid id \cup e \mid \neg id \cap \neg e 
\mid  e \cup \neg e \mid e \cap \neg  e .
$$
%
In order to provide a correspondence with \FOtwo{}, we also introduce  a  modal logic \EML{}, which is similar to $\EMLC$, but the counting parameter $c$ in formulas is always set to $1$. 

As for the semantics of \EMLC{} and \EML{},
we will consider the same models  $\M_{\G} = (V, E, \nu)$ as for other modal logics in the paper.
The valuation function $\nu$ extends to complex formulas as follows:
\begin{align*}
\nu(\neg \varphi) & := V \setminus \nu(\varphi),
\\
\nu(\varphi_1 \land , \varphi_2)  & := \nu(\varphi_1) \cap \nu(\varphi_2), \\
\nu(\langle S \rangle^{\geq k} \varphi) & := \{ v \mid
 k \leq | \{ w \mid \{ v,w \} \in \varepsilon_S \text{ and } w \in \nu(\varphi) \}| \},
\end{align*}
where $\varepsilon_S$ is defined depending on the form of $S$ as follows:

\begin{itemize}
\item $\varepsilon_{id} = \{\{ v,v \} \;| \; v \in V\}$, 
\item $\varepsilon_{\neg id} = \{\{v,w \} \mid v, w \in V\}  \setminus \varepsilon_{id}$,
\item $\varepsilon_e = E$, 
\item $\varepsilon_{\neg e} = \{\{v,w \} \mid v,w \in V\} \setminus \varepsilon_{e}$,
\item $\varepsilon_{id \cup e} = \varepsilon_{id} \cup \varepsilon_{e}$,
\item $\varepsilon_{\lnot id \cap \lnot e} = \varepsilon_{\lnot id} \cap \varepsilon_{\lnot e}$,
\item $\varepsilon_{e \cup \lnot e} = \varepsilon_{e} \cup \varepsilon_{\lnot e}$,
\item $\varepsilon_{e \cap \lnot e} = \varepsilon_{e} \cap \varepsilon_{\lnot e} =  \emptyset$.
\end{itemize}
As usual, a pointed model $(\mathfrak{M}, v)$, 
\emph{satisfies} a formula $\varphi$, denoted $(\mathfrak{M},v) \models \varphi$, if
$v \in \nu(\varphi)$.

\linalgacn*

\begin{proof}
It is known that  
\Ctwo{} over coloured graphs has the same expressivity as as  \EMLC{} \cite[Theorem D.3]{DBLP:conf/iclr/BarceloKM0RS20}---see also \cite[Theorem 1]{DBLP:conf/csl/LutzSW01}.
Hence, each $\Ctwo$ classifier can be written as an equivalently \EMLC{} formula.
Hence, we assume that $\varphi$ is
an $\EMLC$  classifier  of dimension $d$ with propositions $\prop(\varphi)$ and subformulas $\mathsf{sub}(\varphi) = (\varphi_1, \ldots, \varphi_L)$, where $k \leq \ell$ whenever $\varphi_k$ is a subformula of $\varphi_{\ell}$.
We construct a GNN classifier $\N_{\varphi}$ with layers $0, \ldots, L$.
The classification function $\cls$ is such that $\cls(\mathbf{x}) = \true$ iff the last element of $\mathbf{x}$ is $1$.
Layer $0$  uses a combination function to multiply input vectors by a  matrix $\mathbf{D} \in \R^{d \times L}$, namely  $\lambda(v)^{(1)} =  \lambda(v)\mathbf{D} $, where $\mathbf{D}_{k \ell} = 1$ if the $k$th position of the  input vectors  corresponds to a proposition $\varphi_\ell$;
other entries of $\mathbf{D}$ are $0$.
Other layers are of dimension $L$ and defined below. All these remaining layers are homogeneous, meaning that the aggregations and  combination  functions in each case are the same (i.e., they share the same parameters across all layers). The activation function in all cases is the componentwise truncated ReLU defined as $\sigma(x) = \mathit{min}(\mathit{max}(0,x),1)$.

We define  \acn{} layers  with $\comb(\mathbf{x}, \mathbf{y}, \mathbf{z}) = \sigma(\mathbf{x} \mathbf{C}  + \mathbf{y} \mathbf{A} + \mathbf{z} \mathbf{\overline{A}} + \mathbf{b})$ and take aggregation functions, $\agg$ and $ \overline{\agg}$ to be the max-$n$-sum function, where $n$ is the counting rank of $\varphi$ 
(defined similarly as for other counting logics in the paper, namely as the maximal number $c$ occurring in expressions $\langle S \rangle ^{\geq c}$ in $\varphi$). 
We define the entries of $\mathbf{A}, \mathbf{\overline{A}}, \mathbf{C} \in \mathbb{R}^{L \times L}$ and  $\mathbf{b} \in \mathbb{R}^L$ depending on the subformulas of $\varphi$ as follows:

\begin{center}
\begin{tabular}{l}
1. if $\varphi_{\ell}$ is a proposition, $C_{\ell\ell} = 1$, \\
2. if $\varphi_{\ell} = \varphi_j \wedge \varphi_k$, then $C_{j \ell} = C_{k \ell} =1$ and $b_{\ell} = -1$; \\
3. if $\varphi_{\ell} =\neg \varphi_k$, then $C_{k\ell} = -1$ and $b_{\ell} = 1$;\\
4. if $\varphi_{\ell} = \langle id \rangle^{\geq c} \varphi_k$ and $c=1$, then  $C_{k\ell} = 1$
\\
5. if $\varphi_{\ell} = \langle \lnot id \rangle^{\geq c} \varphi_k$ then $A_{k\ell} = 1$, $\overline{A}_{k\ell} = 1$ and $b_{\ell} = -c+1$;\\
6. if $\varphi_{\ell} = \langle e \rangle^{\geq c} \varphi_k$ then $A_{k\ell} = 1$ and $b_{\ell} = -c+1$;\\
7. if $\varphi_{\ell} = \langle \lnot e \rangle^{\geq c} \varphi_k$ then $C_{k\ell} = 1$, $\overline{A}_{k\ell} = 1$ and $b_{\ell} = -c+1$;\\
8. if $\varphi_{\ell} = \langle e \cup id \rangle^{\geq c} \varphi_k$ then $C_{k\ell} = 1$, $A_{k\ell} = 1$ and $b_{\ell} = -c+1$;\\
9. if $\varphi_{\ell} = \langle \lnot e \cap \lnot id \rangle^{\geq c} \varphi_k$ then $\overline{A}_{k\ell} = 1$ and $b_{\ell} = -c+1$;\\
10. if $\varphi_{\ell} = \langle e \cup \lnot e \rangle^{\geq c} \varphi_k$ then $C_{k\ell} = 1$, $A_{k\ell} = 1$, $\overline{A}_{k\ell} = 1$ and $b_{\ell} = -c+1$;
\end{tabular}
\end{center}

\noindent all other entries in the $\ell$-th columns are set to 0. In particular, notice that if  $\varphi_{\ell} = \langle e \cap \lnot e \rangle^{\geq c} \varphi_k$, then all values in  the $\ell$-th columns are set to 0.

Consider the GNN application to $\G = (V,E,\emb)$ of dimension $d$.
We  show that for each subformula $\varphi_{\ell}$, each  $i \in \{\ell, \ldots, L\}$, and  $v \in V$, if   $(\M_{\G},v) \models \varphi_{\ell}$ then $\emb(v)^{(i)}_{\ell} =1 $, and  otherwise $\emb(v)^{(i)}_{\ell} =0$. 
This implies that $\N_\varphi(\G,v) = \true$ iff $(\M_{\G},v) \models \varphi$.
The proof is by induction on the structure of $\varphi_\ell$.

\begin{itemize}
\item If $\varphi_{\ell}$ is a proposition then, by the design of layer $0$, we have  $\lambda(v)^{(0)}_\ell =1$ if $(\M_{G},v) \models \varphi_\ell$, and otherwise $\lambda(v)^{(0)}_\ell =0$.
Moreover, since $\varphi_\ell$ is a proposition,
all \acn{} layers have $C_{\ell \ell} = 1$, $b_{\ell} = 0$, $A_{k\ell} = 0$, and $\overline{A}_{k\ell} = 0$ for each $k$, so
$\comb(\mathbf{x}, \mathbf{y}, \mathbf{z})_{\ell} = \mathbf{x}_{\ell}$.
Thus $\lambda(v)^{(0)}_\ell = \lambda(v)^{(i)}_\ell$ for any $i$.

\item If $\varphi_{\ell} = \neg \varphi_k$, then by construction, we have $C_{k\ell} = -1$, $b_\ell = 1$, $\overline{A}_{m\ell} = 0$, and $A_{m\ell} = 0$ for each $m$. 
Then by Equation~\eqref{eq:acn}, we have $\emb(v)^{(i)}_\ell = \sigma(-\emb(v)^{(i-1)}_k + 1)$.
By inductive hypothesis, 
$\lambda(v)^{(i-1)}_k = 1$ if  $(\M_{G},v) \models \varphi_k$,  and  otherwise $\lambda(v)^{(i-1)}_k = 0$. 
Hence, 
$\emb(v)^{(i)}_\ell=1$ if $(\M_{\G} , v) \models  \neg \varphi_k$, and otherwise $\emb(v)^{(i)}_\ell=0$ for all $i \geq \ell$.

\item If $\varphi_{\ell} = \varphi_j \wedge \varphi_k$, then by construction, we have $C_{j\ell} = C_{k\ell} = 1$, $b_\ell = -1$, $\overline{A}_{m\ell} = 0$, and $A_{m\ell} = 0$ for each $m$.
Then by Equation~\eqref{eq:acn}, we have $\emb(v)^{(i)}_\ell = \sigma(\emb(v)^{(i-1)}_j + \emb(v)^{(i-1)}_k - 1)$.
By inductive hypothesis, 
$\lambda(v)^{(i-1)}_k = 1$ if  $(\M_{G},v) \models \varphi_k$,  and  otherwise $\lambda(v)^{(i-1)}_k = 0$; and $\lambda(v)^{(i-1)}_j = 1$ if  $(\M_{G},v) \models \varphi_j$, and  otherwise $\lambda(v)^{(i-1)}_j = 0$.
Hence, 
$\emb(v)^{(i)}_\ell=1$ if $(\M_{\G} , v) \models  \varphi_j \wedge \varphi_k$, and otherwise $\emb(v)^{(i)}_\ell=0$ for all $i \geq \ell$.

\item If $\varphi_\ell =  \langle id \rangle^{\geq c} \varphi_k$, we obviously cannot have $c \geq 2$ since there is only a single identity element. Then by construction, $C_{k\ell} = 1$ if $c=1$ and $0$ otherwise, and $b_{\ell} = 0$, $\overline{A}_{m\ell} = 0$, and $A_{m\ell} = 0$ for each $m$, so by the inductive hypothesis and Equation~\eqref{eq:acn}, we have $\emb(v)^{(i)}_\ell= \sigma(\emb(v)^{(i-1)}_k) = 1$ if $(\M_{\G} , v) \models \langle id \rangle^{\geq c} \varphi_k$, and otherwise $\emb(v)^{(i)}_\ell=0$ for all $i \geq \ell$.

\item If $\varphi_{\ell} = \langle \lnot id \rangle^{\geq c} \varphi_k$, then by construction, we have $\overline{A}_{k\ell} = 1$, $A_{k\ell} = 1$, $b_\ell = -c+1$ and all remaining entries of the $\ell$-th columns are $0$. Then, by the inductive hypothesis and Equation~\eqref{eq:acn}, we have $\emb(v)^{(i)}_\ell= \sigma( \mathit{min}(n,\sum_{w \in N_G}\emb(w)^{(i-1)}_k) + \mathit{min}(n, \sum_{w \in \overline{N}_G}\emb(w)^{(i-1)}_k) -c + 1) = 1$ if $(\M_{\G} , v) \models \langle \lnot id \rangle^{\geq c} \varphi_k$, and otherwise $\emb(v)^{(i)}_\ell=0$ for all $i \geq \ell$. 

\item If $\varphi_\ell = \langle e \rangle^{\geq c} \varphi_k$, then by construction, we have $A_{k\ell} = 1$, $b_\ell = -c+1$ and all remaining entries of the $\ell$-th columns are $0$. Then, by the inductive hypothesis and Equation~\eqref{eq:acn}, we have $\emb(v)^{(i)}_\ell= \sigma(\mathit{min}(n,\sum_{w \in N_G(v)}\emb(w)^{(i-1)}_k) -c + 1) = 1$ if $(\M_{\G} , v) \models \langle e \rangle^{\geq c} \varphi_k$, and otherwise $\emb(v)^{(i)}_\ell=0$ for all $i \geq \ell$. 

\item If $\varphi_{\ell} = \langle \lnot e \rangle^{\geq c} \varphi_k$, then by construction, we have $\overline{A}_{k\ell} = 1$, $C_{k\ell} = 1$, $b_\ell = -c+1$ and all remaining entries of the $\ell$-th columns are $0$. Then, by the inductive hypothesis and Equation~\eqref{eq:acn}, we have $\emb(v)^{(i)}_\ell= \sigma( \emb(v)^{(i-1)}_k + \mathit{min}(n,\sum_{w \in \overline{N}_G}\emb(w)^{(i-1)}_k) -c + 1) = 1$ if $(\M_{\G} , v) \models \langle \lnot e \rangle^{\geq c} \varphi_k$, and otherwise $\emb(v)^{(i)}_\ell=0$ for all $i \geq \ell$. 

\item If $\varphi_\ell = \langle e \cup id \rangle^{\geq c} \varphi_k$, then by construction, we have $C_{k\ell} = 1$, $A_{k\ell} = 1$, $b_\ell = -c+1$ and all remaining entries of the $\ell$-th columns are $0$. Then, by the inductive hypothesis and Equation~\eqref{eq:acn}, we have $\emb(v)^{(i)}_\ell= \sigma( \emb(v)^{(i-1)}_k + \mathit{min}(n,\sum_{w \in N_G}\emb(w)^{(i-1)}_k) -c + 1) = 1$ if $(\M_{\G} , v) \models \langle e \cup id \rangle^{\geq c} \varphi_k$, and otherwise $\emb(v)^{(i)}_\ell=0$ for all $i \geq \ell$. 

\item If $\varphi_\ell = \langle \lnot e \cap \lnot id \rangle^{\geq c} \varphi_k$, then by construction, we have $\overline{A}_{k\ell} = 1$, $b_\ell = -c+1$ and all remaining entries of the $\ell$-th columns are $0$. Then, by the inductive hypothesis and Equation~\eqref{eq:acn}, we have $\emb(v)^{(i)}_\ell= \mathit{min}(\sigma(n,\sum_{w \in \overline{N}_G}\emb(w)^{(i-1)}_k) -c + 1) = 1$ if $(\M_{\G} , v) \models \langle \lnot e \cap \lnot id \rangle^{\geq c} \varphi_k$, and otherwise $\emb(v)^{(i)}_\ell=0$ for all $i \geq \ell$. 

\item If $\varphi_{\ell} = \langle e \cup \lnot e \rangle^{\geq c} \varphi_k$, then by construction, we have $A_{k\ell} = 1$, $\overline{A}_{k\ell} = 1$, $C_{k\ell} = 1$, $b_\ell = -c+1$ and all remaining entries of the $\ell$-th columns are $0$. Then, by the inductive hypothesis and Equation~\eqref{eq:acn}, we have $\emb(v)^{(i)}_\ell= \sigma( \emb(v)^{(i-1)}_k + \mathit{min}(n,\sum_{w \in N_G}\emb(w)^{(i-1)}_k) + \mathit{min}(n,\sum_{w \in \overline{N}_G}\emb(w)^{(i-1)}_k) -c + 1) = 1$ if $(\M_{\G} , v) \models \langle e \cup \lnot e \rangle^{\geq c} \varphi_k$, and otherwise $\emb(v)^{(i)}_\ell=0$ for all $i \geq \ell$.

\item If $\varphi_{\ell} = \langle e \cap \lnot e \rangle^{\geq c} \varphi_k$, then the induced accessibility relation is empty, so all entries of the $\ell$-th columns are $0$. Hence, we have $\emb(v)^{(i)}_\ell= \sigma(0) = 0$ for all $i \geq \ell$.
\end{itemize}

As for $\FOtwo$, it has the same expressive power as $\EML$.
Hence, by setting $c = 1$ and using the component-wise maximum function in our construction,  we obtain that $\FOtwo \leq \GNN{\s}{\acn}$. 
\end{proof}

\begin{example*}[$\GNN{\bm}{\acn}$ Construction for $\Ctwo$]
We illustrate the construction of $\N \in \GNN{\bm}{\acn}$ from the proof of Theorem \ref{lem:twovar-to-gnn} for $\Ctwo$ formula 
$$
\varphi(x) = \lnot \exists_3 y (R(x,y) \land P(y)) \land \exists_3 x(P(x) \land \exists_3 y (\lnot R(x,y) \land P(y)))
$$ 
with quantifier rank $\ell = 2$ and counting rank $c = 3$. 
We first rewrite $\varphi$ equivalently in $\EMLC$ as 
$$
\varphi = \lnot \langle e \rangle^{\geq 3} p \land \langle e \cup \lnot e \rangle^{\geq 3} (p \land \langle \lnot e \rangle^{\geq 3} p).
$$
Next, we provide an ordering of subformulas of $\varphi$ in $\EMLC$ as described in the proof of \Cref{lem:twovar-to-gnn}, for example:
$\varphi_1 = p$, $\varphi_2 = \langle e \rangle^{\geq 3} p$, $\varphi_3 = \langle \lnot e \rangle^{\geq 3} p$, $\varphi_4 = \lnot \langle e \rangle^{\geq 3} p$, $\varphi_5 = p \land \langle \lnot e \rangle^{\geq 3} p$, $\varphi_6 = \langle e \cup \lnot e \rangle^{\geq 3} (p \land \langle \lnot e \rangle^{\geq 3} p)$, $\varphi_7 = \varphi = \lnot \langle e \rangle^{\geq 3} p \land \langle e \cup \lnot e \rangle^{\geq 3} (p \land \langle \lnot e \rangle^{\geq 3} p)$.

As before, we construct matrix $\mathbf{D}$ to extend the input labels:
$$\mathbf{D} = 
\begin{bmatrix}
1 & 0 & 0 & 0 & 0 & 0 & 0
\end{bmatrix}.
$$

In comparison with $\acr$ architectures, in $\acn$ architectures, the matrix $\mathbf{\overline{A}}$ is responsible for propagating the aggregation over non-neighbours of the current node excluding the current node instead of the set of all nodes. Observe that the modality $\langle e \rangle$ corresponds to propagating the neighbours of the current node (i.e. matrix $\mathbf{A}$ only); the modality $\langle e \cup \lnot e \rangle$ corresponds to all nodes (i.e. matrices $\mathbf{C}, \mathbf{A}$, and $\mathbf{\overline{A}}$); and the modality $\langle \lnot e \rangle$ corresponds to all non-neighbours of the current node including the current node (i.e. matrices $\mathbf{C}$ and $\mathbf{\overline{A}}$ only). Then, by construction from Theorem \ref{lem:twovar-to-gnn}, we get:

$$\mathbf{C} = 
\begin{bmatrix}
1 & 0 & 1 & 0 & 1 & 0 & 0\\
0 & 0 & 0 & -1 & 0 & 0 & 0\\
0 & 0 & 0 & 0 & 1 & 0 & 1\\
0 & 0 & 0 & 0 & 0 & 0 & 0\\
0 & 0 & 0 & 0 & 0 & 1 & 0\\
0 & 0 & 0 & 0 & 0 & 0 & 1 \\
0 & 0 & 0 & 0 & 0 & 0 & 0
\end{bmatrix}, ~~
\mathbf{A} = 
\begin{bmatrix}
0 & 1 & 0 & 0 & 0 & 0 & 0\\
0 & 0 & 0 & 0 & 0 & 0 & 0\\
0 & 0 & 0 & 0 & 0 & 0 & 0\\
0 & 0 & 0 & 0 & 0 & 0 & 0\\
0 & 0 & 0 & 0 & 0 & 1 & 0\\
0 & 0 & 0 & 0 & 0 & 0 & 0\\
0 & 0 & 0 & 0 & 0 & 0 & 0
\end{bmatrix}, ~~
\mathbf{\overline{A}} = 
\begin{bmatrix}
0 & 0 & 1 & 0 & 0 & 0 & 0\\
0 & 0 & 0 & 0 & 0 & 0 & 0\\
0 & 0 & 0 & 0 & 0 & 0 & 0\\
0 & 0 & 0 & 0 & 0 & 0 & 0\\
0 & 0 & 0 & 0 & 0 & 1 & 0\\
0 & 0 & 0 & 0 & 0 & 0 & 0\\
0 & 0 & 0 & 0 & 0 & 0 & 0
\end{bmatrix}, \text{ and }
$$

$$
\mathbf{b} = 
\begin{bmatrix}
0 & -2 & -2 & 1 & -1 & -2 & -1
\end{bmatrix}.
$$

\end{example*}

\ctwogame*

\begin{proof}
We start by showing the first implication by strong induction on $\ell$, and for an arbitrary fixed $c$.
We show a slightly stronger claim that for any $i \in \{ 1, 2\}$ and tuples $\mathbf{a} \in V^i$, $\mathbf{a}' \in V'^i$ of length $i$, 
if $(\mathfrak{M}, \mathbf{a}) \sim_{\ell,c}^{2} (\mathfrak{M}',\mathbf{a}')$ (i.e. if Duplicator has a winning strategy from the game position $(\mathfrak{M},a_1,a_2,\mathfrak{M'}, a'_1,a'_2)$ for tuples $\mathbf{a} = (a_1,a_2)$ and $\mathbf{a'} = (a'_1,a'_2)$),  then
$\mathfrak{M} \models \varphi(\mathbf{a})$ iff $\mathfrak{M}' \models \varphi(\mathbf{a}')$  
for all $\Ctwo_{\ell,c}$ formulas $\varphi(\mathbf{x})$ with $i$ free variables.

For the base case, observe that $(\mathfrak{M}, \mathbf{a}) \sim^{2}_{0,c} (\mathfrak{M'}, \mathbf{a'})$ implies that $\mathbf{a}$ and $\mathbf{a'}$ satisfy the same unary and binary predicates: in the sense that $a_1$ and $a'_1$ satisfy the same unary predicates, $a_2$ and $a'_2$ satisfy the same unary predicates, and $(a_1,a_2)$ and $(a'_1,a'_2)$ satisfy the same binary predicates (including $E$ and $=$).
Hence,  $\mathbf{a}$ in $\mathfrak{M}$ and $\mathbf{a'}$ in $\mathfrak{M'}$ satisfy the same $\Ctwo_{0,c}$ formulas, which are simply Boolean combinations of unary and binary predicates.
For the inductive hypothesis, suppose  for any $i \in \{ 1, 2\}$ and tuples $\mathbf{a} \in V^i$, $\mathbf{a}' \in V'^i$,  
if $(\mathfrak{M}, \mathbf{a}) \sim_{\ell,c}^{2} (\mathfrak{M}',\mathbf{a}')$, then
$\mathfrak{M} \models \varphi(\mathbf{a})$ iff $\mathfrak{M}' \models \varphi(\mathbf{a}')$  
for all $\Ctwo_{\ell,c}$ formulas $\varphi(\mathbf{x})$ with $i$ free variables.
For the inductive step, 
assume that there is $\varphi(\mathbf{x}) \in \Ctwo_{\ell+1,c}$ such that $\mathfrak{M} \models \varphi(\mathbf{a})$, but $\mathfrak{M}' \not\models \varphi(\mathbf{a'})$. 
We will show that 
 $(\mathfrak{M}, \mathbf{a}) \not\sim^{2}_{\ell+1,c} (\mathfrak{M'}, \mathbf{a'})$.
 
If $\varphi(\mathbf{x}) \in \Ctwo_{\ell,c}$, the result holds by the inductive hypothesis. 
Hence, we will focus on the case when $\varphi(x) = \exists_{n}y. \psi(x,y) $ with $n \leq c$ and $\psi(x,y) \in \Ctwo_{\ell,c}$ (the cases when $\varphi$ is a conjunction or a negated formula, are easy to show).
Since  $\mathfrak{M} \models \varphi(a)$, there exists $U = \{b_1, \ldots, b_n \}$  such that $\mathfrak{M} \models \psi(a,b_j)$, for each $b_j$.
We will show that $(\mathfrak{M}, a) \not\sim_{\ell+1, c}^{2} (\mathfrak{M}',a')$.
The winning strategy for Spoiler is to choose $\M$, the second pair of pebbles, and the set $U$. Because 
$\mathfrak{M}' \not\models \exists_{n}y. \psi(a',y) $, whichever set $U'$ of size $n$ Duplicator chooses, there exists $b' \in U'$  such that
$\mathfrak{M}' \not\models \psi(a',b')$. 
Spoiler places  $p_{\mathfrak{M}'}^2$ on $b'$.
Duplicator needs to place $p_{\mathfrak{M}}^2$ on some $b \in U$, and the  game is in configuration
$(\M,a,b, \M',a', b')$.
Since $\mathfrak{M} \models \psi(a,b)$ and $\mathfrak{M}' \not\models \psi(a',b')$, by the inductive hypothesis, we obtain that
$(\mathfrak{M}, a,b) \not\sim_{\ell,c}^{2} (\mathfrak{M}',a',b')$.
Thus,  $(\mathfrak{M}, \mathbf{a}) \not\sim^{2}_{\ell+1,c} (\mathfrak{M'}, \mathbf{a'})$, as required.

For the other direction, we show the following implication for any $i \in \{ 1, 2\}$ and tuples $\mathbf{a} \in V^i$, $\mathbf{a}' \in V'^i$.  
If $\mathfrak{M} \models \varphi(\mathbf{a})$ iff $\mathfrak{M}' \models \varphi(\mathbf{a}')$  
for all $\Ctwo_{\ell,c}$ formulas $\varphi(\mathbf{x})$ with $i$ free variables, then    $(\mathfrak{M}, \mathbf{a}) \sim_{\ell,c}^{2} (\mathfrak{M}',\mathbf{a}')$.
We again prove it by induction on $\ell$, with a fixed $c$.
In the base, for $\ell = 0$, if  $\mathfrak{M} \models \varphi(\mathbf{a})$ iff $\mathfrak{M'} \models \varphi(\mathbf{a'})$ where $\varphi(\mathbf{x}) \in \Ctwo_{0,c}$, then they satisfy the same unary and binary predicates, so we have $(\mathfrak{M}, \mathbf{a}) \sim^2_{0,c} (\mathfrak{M'}, \mathbf{a'})$.
In the inductive step we need to show a winning strategy for Duplicator. 
Assume that Spoiler starts by choosing $\M$, $i=1$ (i.e. a pebble on $\mathbf{a}_1$), and a set $U = \{u_1, \dots, u_n \}$ of $n \leq c$ elements in $\mathfrak{M}$. 
Before we define the set $U'$ that  Duplicator should choose, we
observe that up to the equivalence there are finitely many $\Ctwo_{\ell,c}$ formulas with at most two free variables \cite[Lemma 4.4]{cai1992optimal}.
We call  conjunctions of such formulas  \emph{types} and  we let $T$ be the set of all types. 
We let the type, $t(\M,\mathbf{a})$, of  elements $\mathbf{a}$ in $\mathfrak{M}$   be  the (unique) maximal length 
$\psi(\mathbf{x}) \in T$  such that $\mathfrak{M} \models \psi(\mathbf{a})$.
Duplicator should  choose $U' =\{u'_1, \dots, u'_n \}$ of elements in $\mathfrak{M}'$ such that $t(\M', \mathbf{a}'[\mathbf{a}_1' \mapsto u'_j]) = t(\M, \mathbf{a}[\mathbf{a}_1 \mapsto u_j])$, for each $u_j$, where $\mathbf{a}'[\mathbf{a}_1' \mapsto u'_j]$ is obtained from $\mathbf{a}'$ by replacing $\mathbf{a}_1'$ with $u'_j$.
We observe that such $U'$ needs to exist.
Otherwise, w.l.o.g.
$k$ among
$\mathbf{a}[\mathbf{a}_1 \mapsto u_j]$ in $\M$ have some  type $\psi(\mathbf{x}) \in T$, but less than $k$ among  $\mathbf{a}'[\mathbf{a}'_1 \mapsto u'_j]$ have this type in $\M'$.
Hence, $\mathfrak{M}  \models \exists_k \mathbf{x}_1 \psi(\mathbf{x}) [\mathbf{x}_2 \mapsto \mathbf{a}_2]$, but 
$\mathfrak{M}' \not \models \exists_k \mathbf{x}_1 \psi(\mathbf{x}) [\mathbf{x}_2 \mapsto \mathbf{a}'_2 ]$.
Since $k \leq  n \leq c$, this is a $\Ctwo_{\ell,c}$ formula, which contradicts the assumption of the inductive step.
Thus,  $U'$  must exist.

Now, assume that Spoiler chooses  $u'_j \in U'$.
Duplicator should choose $u_j \in U$. Since 
$t(\M', \mathbf{a}'[\mathbf{a}_1' \mapsto u'_j]) = t(\M, \mathbf{a}[\mathbf{a}_1 \mapsto u_j])$, we obtain that $\mathbf{a}[\mathbf{a}_1 \mapsto u_j]$ in $\mathfrak{M}$ satisfies the same $\Ctwo_{\ell,c}$ formulas as $\mathbf{a}'[\mathbf{a}'_1 \mapsto u'_j]$ in $\mathfrak{M}'$.
By inductive assumption, 
$(\mathfrak{M}, \mathbf{a}[\mathbf{a}_1 \mapsto u_j]) \sim_{\ell,c}^{2} (\mathfrak{M}',\mathbf{a}'[\mathbf{a}'_1 \mapsto u'_j])$.
Thus, $(\mathfrak{M}, \mathbf{a}) \sim_{\ell+1,c}^{2} (\mathfrak{M}',\mathbf{a}')$.

Now, we will show that a class of pointed models closed under $\sim_{\ell,c}^{2}$ is definable by a (finite) disjunction of  $\Ctwo_{\ell,c}$ formulas.
We have showed that $(\mathfrak{M}, a) \sim^{2}_{\ell,c} (\mathfrak{M}',a')$ iff $a$ in $\mathfrak{M}$ and $a'$ in $\mathfrak{M}'$ satisfy the same $\Ctwo_{\ell,c}$ formulas with one free variable.
Since, up to logical equivalence,  there are finitely many $\Ctwo_{\ell,c}$ formulas, each $\sim_{\ell,c}^{2}$ equivalence class over pointed models can be  expressed as a (finite) disjunction of (finite) $\Ctwo_{\ell,c}$ formulas.
\end{proof}

\biacn*

\begin{proof}
We first consider $\N \in \GNN{\bm}{\acn}$.
We show by induction on $\ell \leq L$ and fixed $k$ that $(G_1,v_1) \sim_{\ell,k}^{2} (G_2, v_2)$ implies  $\emb_1(v_1)^{(\ell)} = \emb_2(v_2)^{(\ell)}$,
where $G_1$, $G_2$ are graphs, $v_1$, $v_2$ any of their nodes, and $k \geq 1$. 
If $\ell = 0$, 
then $v_1$ and $v_2$ satisfy the same propositions, so $\emb_1(v_1)^{(0)} = \emb_2(v_2)^{(0)}$.
For the induction step, assume $(G_1,v_1) \sim_{\ell-1,k}^{2} (G_2, v_2)$ implies $\emb_1(v_1)^{(\ell-1)} = \emb_2(v_2)^{(\ell-1)}$. 
We assume that $\emb_1(v_1)^{(\ell)} \neq \emb_2(v_2)^{(\ell)}$ and we will show that  $(G_1,v_1) \not\sim_{\ell,k}^{2} (G_2, v_2)$.
%
%
%
Since $\emb_1(v_1)^{(\ell)} \neq \emb_2(v_2)^{(\ell)}$, by Equation \eqref{eq:acn}, one of the following statements needs to hold:
\begin{center}
\begin{tabular}{l}
1. $\emb_1(v_1)^{(\ell-1)} \neq \emb_2(v_2)^{(\ell-1)}$,\\
2. $\agg(\lBrace   \emb_1(w)^{(\ell -1)}  \rBrace _{w \in N_{\G_1}(v_1)}) \neq \agg(\lBrace   \emb_2(w)^{(\ell-1)}  \rBrace _{w \in N_{\G_2}(v_2)})$,  \\
3. $\overline{\agg}(\lBrace   \emb_1(w)^{(\ell-1)}  \rBrace _{w \in \overline{N}_{\G_1}(v_1)}) \neq \agg(\lBrace   \emb_2(w)^{(\ell-1)}  \rBrace _{w \in \overline{N}_{\G_2}(v_2)})$.
\end{tabular}
\end{center}
Statement (1) and the inductive hypothesis immediately imply that $(G_1, v_1) \not\sim_{\ell-1,k}^{2} (G_2, v_2)$, so we get $(G_1, v_1) \not\sim_{\ell,k}^{2} (G_2, v_2)$. 

Next, we consider Statements (2) and (3). 
The $k$-boundedness of $\agg$  and $\overline{\agg}$ implies that for each $X \in \{N,\overline{N}\}$ there exists $u \in X_{\G_1}(v_1)$  such that $\emb_1(u)^{(\ell-1)}$ appears $k_1$ times in $\lBrace   \emb_1(w)^{(\ell-1)}  \rBrace _{w \in X_{\G_1}(v_1)}$ but $k_2 \neq k_1$ times in $\lBrace   \emb_2(w)^{(\ell-1)}  \rBrace _{w \in X_{\G_2}(v_2)}$ such that $k_1 < k$ or $k_2 < k$. 
Assume w.l.o.g. that $k_2 < k_1$, so $k_2 <k$. 
We will show that Spoiler has a winning strategy from the configuration $(\mathfrak{M}_{G_1},v_1,\mathfrak{M}_{G_2},v_2)$. 
The strategy for Spoiler is to pick a set $U_1$ of $\mathit{min}(k,k_1)$  elements from $\lBrace   \emb_1(w)^{(\ell-1)}  \rBrace _{w \in X_{\G_1}(v_1)}$ with label $\emb_1(u)^{(\ell-1)}$. 
Then, Duplicator has to respond by picking a set $U_2$ of $\mathit{min}(k,k_1)$ elements from $\lBrace   \emb_2(w)^{(\ell-1)}  \rBrace _{w \in X_{\G_2}(v_2)}$; otherwise partial isomorphism is not satisfied by the positions of the two pairs of pebbles. 
Since $k_2 <k_1$ and $k_2<k$, we have $k_2 < \min(k_1,k)$.
Hence, there must exists $u' \in U_2$ such that $\emb_2(u')^{(\ell-1)} \neq \emb_1(u)^{(\ell-1)}$. 
Now, the strategy for Spoiler is to pick this $u'$.
Indeed, this is because the inductive hypothesis implies $(G_1, u) \not\sim_{\ell-1,k}^{2} (G_2, u')$.
So we can conclude that $(G_1, v_1) \not\sim_{\ell,k}^{2} (G_2, v_2)$. 

It is straightforward to extend to the case of $\GNN{\s}{\acn}$ by considering the absence of a label instead of cardinality discrepancy of a label bounded by $k$. The proof for $\GNN{\bm}{\acn}$ goes through for $\GNN{\s}{\acn}$ under the special case where $k=1$.
\end{proof}

\newpage

\end{document}